\documentclass{article} 
\usepackage[margin=1.05in]{geometry}
\title{Generalization and Stability of Interpolating Neural Networks\\ with Minimal Width
}
\usepackage{amsmath,amssymb,amsfonts}
\usepackage{amsthm}
\usepackage{url} 

\newtheorem{theorem}{Theorem}
\newtheorem{corollary}{Corollary}[theorem]
\newtheorem{lemma}[theorem]{Lemma}
\newtheorem{proposition}[theorem]{Proposition}
\newtheorem{remark}{Remark}
\usepackage{helvet}  
\usepackage{courier}  
\usepackage{graphicx}
\usepackage{textcomp}
\usepackage{xcolor}
\usepackage{setspace}
\usepackage{algpseudocode}
\definecolor{darkred}{RGB}{230,0,0}
\definecolor{darkgreen}{RGB}{0,150,0}
\definecolor{darkblue}{RGB}{0,0,150}

\usepackage{bm} 
\usepackage{color,soul}

\usepackage{microtype}
\usepackage[utf8]{inputenc} 
\usepackage[T1]{fontenc}    
\usepackage{url}            
\usepackage{nicefrac}       
\usepackage{microtype,mathrsfs}
\usepackage{graphicx}
\usepackage{subfigure}
\usepackage{booktabs} 
\usepackage{float}
\usepackage{amsfonts}       
\usepackage{caption}
\usepackage{mathtools,amssymb,comment}
\usepackage{bbm}
\usepackage{breqn} 
\usepackage{enumitem}
\usepackage{tcolorbox}

\usepackage{hyperref}
\usepackage{breakcites}

\definecolor{darkred}{RGB}{150,0,0}
\definecolor{darkgreen}{RGB}{0,120,0}
\definecolor{darkblue}{RGB}{0,0,150}

\hypersetup{ breaklinks=true, 
colorlinks=true, linkcolor=darkred, citecolor=darkgreen, urlcolor=darkblue}

\newcommand{\sbwc}{self-bounded weak convexity }

\usepackage{mathtools}

\DeclarePairedDelimiterX{\inp}[2]{\langle}{\rangle}{#1, #2}

\newcommand{\wpr}{w^\prime}
\newcommand{\wa}{w_\alpha}

\newcommand{\Fh}{\widehat F}

\newcommand{\fp}{f^\prime}
\newcommand{\fpp}{f^{\prime\prime}}
\newcommand{\Lf}{L_f}
\newcommand{\Gf}{G_f}
\newcommand{\lamin}[1]{\lambda_{\min}\left(#1\right)}
\newcommand{\dpr}{d^{\prime}}

\newcommand{\regret}{\texttt{Reg}}
\newcommand{\loo}{\texttt{loo}}

\newcommand{\negi}{{\neg i}}

\newcommand{\GLQC}{generalized-local quasi-convexity~}

\newcommand{\iid}{IID~}

\newcommand{\Phixi}{\Phi(w,x_i)}

\newcommand{\DPhix}{\nabla_1\Phi(w,x)}
\newcommand{\DPhixi}{\nabla_1\Phi(w,x_i)}
\newcommand{\DDPhix}{\nabla^2_1\Phi(w,x)}
\newcommand{\DDPhixi}{\nabla^2_1\Phi(w,x_i)}

\makeatletter
\newcommand*{\rom}[1]{\expandafter\@slowromancap\romannumeral #1@}
\makeatother

\newcommand{\mathleft}{\@fleqntrue\@mathmargin0pt}
\newcommand{\mathcenter}{\@fleqnfalse}

\makeatletter
\newcommand{\ssymbol}[1]{^{\@fnsymbol{#1}}}
\makeatother

\newcommand{\R}{\mathbb{R}}

\newtheorem{defn}{Definition}

\providecommand{\abs}[1]{\lvert#1\rvert}

\newcommand{\hf}{\widehat{F}}

\newtheorem{ass}{Assumption}

\newcommand{\eps}{\varepsilon}

\newcommand{\one}{\mathbf{1}}

\newcommand{\E}{\mathbb{E}}                    

\newcommand{\nn}{\notag}

\newcommand{\Dc}{\mathcal{D}}

\newcommand{\toc}{\tilde{O}}

\newcommand{\beq}{\begin{equation}}
\newcommand{\eeq}{\end{equation}}
\newcommand{\bea}{\begin{align}}
\newcommand{\eea}{\end{align}}
\newcommand{\beas}{\begin{align*}}
\newcommand{\eeas}{\end{align*}}

   \newcommand{\footremember}[2]{%
    \footnote{#2}
    \newcounter{#1}
    \setcounter{#1}{\value{footnote}}%
}

\def\bea#1\eea{\begin{align}#1\end{align}}

\emergencystretch=\maxdimen
\usepackage{cite}

\author{Hossein Taheri\footremember{ht}{Department of Electrical and Computer Engineering, University of California, Santa Barbara. Email: \href{mailto:hossein@ucsb.edu}{hossein@ucsb.edu}} \;\;and\;\; Christos Thrampoulidis\footremember{ct}{Department of Electrical and Computer Engineering, University of British Columbia. Email: \href{mailto:cthrampo@ece.ubc.ca}{cthrampo@ece.ubc.ca}}}

\begin{document}
\maketitle
\begin{abstract}%
We investigate the generalization and optimization properties of shallow neural-network classifiers trained by gradient descent in the interpolating regime. Specifically, in a realizable scenario where model weights can achieve arbitrarily small training error $\epsilon$ and their distance from initialization is $g(\epsilon)$,
we demonstrate that gradient descent with $n$ training data  achieves training error $O(g(1/T)^2\big/T)$ and generalization error $O(g(1/T)^2\big/n)$ at iteration $T$, provided there are at least $m=\Omega(g(1/T)^4)$ hidden neurons. We then show that our realizable setting encompasses a special case where data are separable by the model's neural tangent kernel. For this and logistic-loss minimization, we prove  the training loss decays at a rate of $\tilde O(1/ T)$ given polylogarithmic number of neurons $m=\Omega(\log^4 (T))$. Moreover, with $m=\Omega(\log^{4} (n))$ neurons and $T\approx n$ iterations, we bound the test loss by $\tilde{O}(1/ n)$. Our results differ from existing generalization outcomes using the algorithmic-stability framework, which necessitate polynomial width and yield suboptimal generalization rates. Central to our analysis is the use of a new self-bounded weak-convexity property, which leads to a generalized local quasi-convexity property for sufficiently parameterized neural-network classifiers. Eventually, despite the objective's non-convexity, this leads to convergence and generalization-gap bounds that resemble those found in the convex setting of linear logistic regression.

\end{abstract}


\section{Introduction}
Neural networks have remarkable expressive capabilities and can memorize a complete dataset even with mild overparameterization. In practice, using gradient descent (GD) on neural networks with logistic or cross-entropy loss can result in the objective reaching zero training error and close to zero training loss. Zero training error, often referred to as ``interpolating'' the data, indicates perfect classification of the dataset. Despite their strong memorization ability, these networks also exhibit remarkable generalization capabilities to new data. This has motivated  a surge of studies in recent years exploring the optimization and generalization properties of first-order gradient methods in overparameterized neural networks, with a specific focus in the so-called Neural Tangent Kernel (NTK) regime. In the NTK regime, the model operates as the first-order approximation of the network at a sufficiently large initialization or at the large-width limit \cite{jacot2018neural,chizat2019lazy}. Prior works on this topic mostly focused on quadratic-loss minimization and their optimization/generalization guarantees required network widths that increased polynomially with the sample size $n$. This, however, is not in line with practical experience. Improved results were obtained more recently by \cite{Ji2020Polylogarithmic,chen2020much} who have investigated the optimization and generalization of ReLU neural networks with logistic loss, which is more suitable for classification tasks. Assuming that the NTK with respect to the model can interpolate the data (i.e. separate them with positive margin $\gamma$), they showed through a Rademacher complexity analysis that GD on neural networks with polylogarithmic width can achieve generalization guarantees that decrease with the sample size $n$ at a rate of $\tilde{O}(\frac{1}{\sqrt{n}})$.

\par
In this paper, we provide rate-optimal optimization and generalization analyses of GD for shallow neural networks of minimal width assuming that the model itself can interpolate the data. We focus on two-layer networks with smooth activations that can almost surely separate $n$ training samples from the data distribution. Concretely, we consider a realizability condition where data and initialization are such that model weights can achieve arbitrarily small training error $\eps$ while their distance from initialization is $g(\eps)$ for some function $g:\R_+\rightarrow\R_+$. Under this condition, we demonstrate generalization guarantees of order $O(\frac{g(\frac{1}{T})^2}{n}).$
More generally, for any iteration $T$ of GD and assuming network width $m=\Omega(g(\frac{1}{T})^4)$, we obtain an expected test-loss rate $O(\frac{g(\frac{1}{T})^2}{T} + \frac{g(\frac{1}{T})^2}{n})$. Additional to the generalization bounds, we provide optimization guarantees under the same setting by showing that the training loss approaches zero at rate $O(\frac{g(\frac{1}{T})^2}{T})$. We note that these results are derived without NTK-type analyses. For demonstration and also for connection to prior works on neural-tangent data models, we specialize our generalization and optimization results to the class of NTK-separable data. We show this is possible because the NTK-data separability assumption implies our realizability condition holds. Thus,  for logistic-loss minimization on NTK-separable data, we show that the expected test loss of GD is $\tilde O(\frac{1}{T}+ \frac{1}{ n})$ provided polylogarithmic number of neurons  $m=\Omega(\log^4 (T))$. This further suggests that a network of width $m=\Omega(\log^4 (n))$, attains expected test loss $\tilde O(\frac{1}{n})$ after $T\approx n$ iterations.
\par
In contrast to prior optimization and generalization analyses that often depend on the NTK framework, which requires the first-order approximation of the model, we build on the algorithmic stability approach \cite{bousquet2002stability} for shallow neural-network models of finite width. Although the stability analysis has been utilized in previous studies to derive generalization bounds for (stochastic) gradient descent in various models, most results that are rate-optimal heavily rely on the convexity assumption. Specifically, the stability-analysis framework has been successful in achieving optimal generalization bounds for convex objectives in \cite{lei2020fine,bassily2020stability,pmlr-v178-schliserman22a}. On the other hand, previous studies on non-convex objectives either resulted in suboptimal bounds or relied on assumptions that are not in line with the actual practices of neural network training. For instance, \cite{hardt2016train} derived a generalization bound of $O(\frac{T^{{\beta c}/(\beta c+1)}}{n})$ for general $\beta$-smooth and non-convex objectives, but this required a time-decaying step-size $\eta_t\le c/t$, which can degrade the training performance. More recently, \cite{richards2021learning} explored the use of the stability approach specifically for logistic-loss minimization of a two-layer network. By refining the model-stability analysis framework introduced by \cite{lei2020fine}, they derived generalization-error bounds provided the hidden width increases polynomially with  the sample size. In comparison, our analysis leads to improved generalization and optimization rates and under standard separability conditions such as NTK-separability, only requires a polylogarithmic width for both global convergence and generalization.


\subsection*{Notation}
We denote $[n]:=\{1,2,\cdots,n\}$. We use the standard notation $O(\cdot
),\Omega(\cdot
)$ and use $\tilde{O}(\cdot
),\tilde\Omega(\cdot
)$ to hide polylogarithmic factors. Occasionally we use $\lesssim$ to hide numerical constants. The Gradient and Hessian of a function $\Phi:\R^{d_1\times d_2}\rightarrow \R$ with respect to the $i$th input ($i=1,2$) are denoted by $\nabla_i \Phi$ and $\nabla_i^2 \Phi$, respectively. All logarithms are in base $e.$ We use $\|\cdot\|$ for the $\ell_2$ norm of vectors and the operator norm of matrices. We denote $[w_1,w_2]:=\{w\,:\,w=\alpha w_1+(1-\alpha) w_2, \alpha\in[0,1]\}$ the line segment between $w_1,w_2\in\R^{d'}$.
\section{Problem Setup}\label{sec:problem setup}
Given $n$ i.i.d. samples $(x_i,y_i)\sim \mathcal{D}, i\in[n]$ from data distribution $\Dc$, we study unconstrained empirical risk minimization
with objective $\widehat F:\R^{d^\prime}\rightarrow\R$:
\bea\label{eq:logreg}
\min_{w\in\R^{d^\prime}}\Big\{ \widehat F(w) :=\frac{1}{n}\sum_{i=1}^{n} \hf_i(w)= \frac{1}{n}\sum_{i=1}^{n} f\left(y_i \Phi\left(w,x_i\right)\right)\Big\}.
\eea
 This serves as a proxy for minimizing 
the \emph{test loss}  $F:\R^{d^\prime}\rightarrow\R$:
\bea
F(w):= \E_{(x,y)\sim\mathcal{D}}\left[ f\left(y\Phi(w,x)\right) \right].
\eea
We introduce our assumptions on the data $(x,y)$, the model $\Phi(\cdot,x)$, and the loss function $f(\cdot)$, below. We start by imposing the following mild assumption on the data distribution.

\begin{ass}[Bounded features]\label{ass:features}
Assume any $(x,y)\sim\Dc$ has almost surely bounded features, i.e. $\|x\|\leq R$, and binary label $y\in\{\pm 1\}.$
\end{ass}

The model $\Phi:\R^{d^\prime} \times \R^d\rightarrow\R$ is parameterized by trainable weights $w\in\R^{d'}$ and takes input $x\in\R^d$. For our main results, we assume $\Phi$ is a one-hidden layer neural-net of $m$ neurons, i.e.
\bea\label{eq:model problem setup}
\Phi(w,x):= \frac{1}{\sqrt{m}}\sum_{j=1}^m a_j \,\sigma(\left\langle w_j,x\right\rangle), 
\eea
where $\sigma:\R\rightarrow\R$ is the activation function, $w_j\in\R^d$ denotes the weight vector of the $j$th hidden neuron and $\frac{a_j}{\sqrt{m}}, j\in[m]$ are the second-layer weights. For the second layer weights, we assume that they are fixed during training taking values $a_j\in\{\pm 1\}$. We assume that for half of second layer weights we have $a_j=1$ and for the other half $a_j=-1$. On the other hand, all the first-layer weights are updated during training. Thus, the total number of trainable parameters is $d^\prime=md$ and we denote  $w=[w_1;w_2;\dots;w_m]\in\R^{d^\prime}$  the vector of trainable weights. Throughout, we make the following assumptions on the activation function.

\begin{ass}[Lipschitz and smooth activation] \label{ass:act} The activation function $\sigma:\R\rightarrow\R$ satisfies the following for non-negative constants $\ell, L$:
\bea\nn
\abs{\sigma^\prime(u)}\le\ell,\;\;\;|\sigma^{\prime\prime}(u)| \le L,\qquad\forall u\in\R.
\eea 
\end{ass}

 We note that the smoothness assumption which is required by our framework excludes the use of ReLU. Examples of activation functions that satisfy the smoothness condition include Softplus $\sigma(u)=\log(1+e^u)$, Gaussian error linear unit (GELU) $\sigma(u)=\frac{1}{2} u(1+\operatorname{erf}(\frac{u}{\sqrt{2}}))$, and Hyperbolic-Tangent where $\sigma(u)=\frac{e^u-e^{-u}}{e^u+e^{-u}}$. On the other hand, Lipschitz assumption is rather mild, since it is possible to restrict the parameter space to a bounded domain.

Next, we discuss  conditions on the loss function. Of primal interest is the commonly used logistic loss function $f(u)=\log(1+e^{-u}).$ However, our results hold for a broader class of convex, non-negative and monotonically decreasing functions ($\lim_{u\rightarrow \infty} f(u) = 0$) that  satisfy the following:

\begin{ass}[Lipschitz and smooth loss]\label{ass:lip and smooth}

    The convex loss function $f:\R\rightarrow\R_+$ satisfies for all $u\in\R$
\begin{enumerate}[label={\textbf{\theass.\Alph*:}},
  ref={\theass.\Alph*}]
\item  \label{ass:loss lipschitz}%
    Lipschitzness: $|\fp(u)|\leq \Gf.$
    \item \label{ass:loss smooth}%
    Smoothness: $\fpp(u)\leq \Lf.$
\end{enumerate}
\end{ass}

\begin{ass}[Self-bounded loss]\label{ass:self boundedness}
    The convex loss function $f:\R\rightarrow\R_+$ is self-bounded with some constant $\beta_f>0$, i.e., 
    $|\fp(u)|\leq \beta_f f(u), \forall u\in\R.$

\end{ass}
The self-boundedness Assumption \ref{ass:self boundedness} is the key property of the loss that drives our analysis and justifies the polylogarithmic width requirement, as will become evident. Note that the logistic loss naturally satisfies Assumptions \ref{ass:loss lipschitz} and \ref{ass:loss smooth}  (with $G_f=1,L_f=1/4$), as well as, Assumption \ref{ass:self boundedness} with $\beta_f=1$. Other interesting examples of loss functions satisfying those assumptions include polynomial losses, with the tail behavior $f(u)=1/u^\beta$ for $\beta>0$, which we discuss in Remark \ref{rem:polyn}. To lighten the notation and without loss of generality, we set  $G_f=L_f=\beta_f=1$ for the rest of the paper. We remark that our training-loss results also hold for the exponential loss $e^{-u}$. The exponential loss is self-bounded and while it is not Lipschitz or smooth it satisfies a second-order self-bounded property $\fpp(u)\leq f(u)$, which we can leverage instead; see Appendix \ref{sec:training analysis} for details.

\section{Main Results}

We present bounds on the train loss and generalization gap of gradient-descent (GD) under the setting of Section \ref{sec:problem setup}. Formally, GD with step-size $\eta>0$ optimizes \eqref{eq:logreg} by performing the following updates starting from an initialization $w_0$:
\bea\nn
\forall t\geq 0\,:\,\,w_{t+1}=w_t-\eta\nabla \hf(w_t).
\eea

\subsection{Key properties}\label{sec:properties}

The key challenge in  both the optimization and generalization analysis is the non-convexity of $f(y\Phi(\cdot,x))$, and consequently of the train loss $\Fh(\cdot)$. Despite non-convexity,  we derive bounds analogous to the convex setting, e.g. corresponding bounds on linear logistic regression in \cite{ji2018risk,shamir2021gradient,pmlr-v178-schliserman22a}. We show  this is possible provided the loss satisfies the following key property, which we call \emph{self-bounded weak convexity}.

\begin{defn}[Self-bounded weak convexity]\label{defn:sbhs}
We say a function $\hf:\R^{d^\prime}\rightarrow\R$ is self-bounded weakly convex if there exists constant $\kappa>0$ such that for all $w$,
\bea\label{eq:self-bounded hessian lamin}
\lambda_{\min}\left(\nabla ^2 \widehat F(w)\right) \ge -\kappa \,\widehat F(w)\,.
\eea
\end{defn}
 Recall a function $G:\R^{\dpr}\rightarrow\R$ is weakly convex if $\exists\kappa\geq0$ such that uniformly over all $w\in\R^{d'}$,  $\lamin{\nabla ^2 G(w)}\geq -\kappa.$ If $\kappa=0$,  the function is convex. Instead, property \eqref{eq:self-bounded hessian lamin} lower bounds the curvature by $-\kappa \,G(w)$ that changes proportionally with the function value $G(w)$. We explain below how this is exploited in our setting.

To begin with, the following lemma 
shows that property \eqref{eq:self-bounded hessian lamin} holds for the train loss 
under the setting of Section \ref{sec:problem setup}: training of a two-layer net with smooth activation and self-bounded loss. The lemma also shows that the gradient of the train loss is self bounded. Those two properties together summarize the key ingredients for which our analysis applies.

\begin{lemma}[Key self-boundedness properties]
    \label{lem:key prop intro}
Consider the setup of Section \ref{sec:problem setup} and let Assumptions \ref{ass:features}-\ref{ass:act} hold. Further assume the loss is self-bounded as per Assumption \ref{ass:self boundedness}. Then, the objective satisfies the following self-boundedness properties for its Gradient and Hessian:
\begin{enumerate}
  
\item Self-bounded gradient: $\left\|\nabla \hf_i(w)\right\| \leq \ell R\, \hf_i(w),\;\; \forall i\in[n]$.

\item Self-bounded weak convexity: $\lambda_\min\left(\nabla^2 \widehat F(w)\right) \geq - \frac{LR^2}{\sqrt{m}}\widehat F(w)$.

\end{enumerate}
\end{lemma}

Both of these properties follow from the self-boundedness of the convex loss $f$ combined with Lipshitz and smoothness of $\sigma$. The self-boundedness of the gradient is used for generalization analysis and in particular in obtaining the model stability bound. The \sbwc plays an even more critical role for our optimization and generalization results.
 In particular, the wider the network the closer the loss to having convex-like properties. Moreover, the ``self-bounded'' feature of this property provides another mechanism that favors convex-like optimization properties of the loss. To see this, consider the minimum Hessian eigenvalue $\lambda_\min(\nabla^2\Fh(w_t))$ at gradient descent iterates $\{w_t\}_{t\geq 1}$: As training progresses, the train loss $\Fh(w_t)$ decreases, and thanks to the \sbwc property, the gap to convexity also decreases. We elaborate on the role of self-bounded weak convexity in our proofs in Section \ref{sec:ps}.

\subsection{Training loss}\label{sec:training loss}
We begin with a general bound on the training loss and the parameter's norm, which is also required for our generalization analysis.

\begin{theorem}[Training loss -- General bound] \label{thm:train} Suppose Assumptions \ref{ass:features}-\ref{ass:self boundedness} hold. 
Fix any training horizon $T\geq 0$ and any step-size $\eta\le 1/L_\hf$ where $L_\hf$ is the objective's smoothness parameter. Assume any $w\in\R^{d^\prime}$ and hidden-layer width $m$ such that
 $\|w-w_0\|^2\geq \max\{\eta T \widehat F(w),\eta \widehat F(w_0)\}$ and $m \geq 18^2 L^2 R^4\|w-w_0\|^4$.
Then, the training loss and the parameters' norm satisfy
\begin{align}\label{eq:train loss bound final}
        &\hf(w_T) \;\leq\; \frac{1}{T}\sum_{t=1}^T\Fh(w_t) \;\leq\; 2 \hf(w) + 
        \frac{5\|w-w_0\|^2}{2\eta T},
        \\
        &\forall t\in[T]\;\; :\;\; \left\|w_t-w_0\right\|\;\le\; 4\|w-w_0\|.\nn
\end{align}
\end{theorem}
A few remarks are in place regarding the theorem.  First,  Eq. \eqref{eq:train loss bound final} upper bounds the running average (also known as regret) of train loss for iterations $1,\ldots,T$ by the value, at an arbitrarily  chosen point $w$, of a ridge-regularized objective with regularization parameter inversely proportional to $\eta T$. Because of smoothness and Lipschitz Assumption \ref{ass:lip and smooth} of $f$, it turns out that the training objective is $L_\hf$-smooth.
Hence, by the descent lemma of GD for smooth functions, the same upper bound holds in Eq. \eqref{eq:train loss bound final} for the value of the loss at time $T$, as well.  Moreover, the theorem provides a uniform upper bound of the norm of all GD iterates in terms of $\|w-w_0\|$. Notably, and despite the non-convexity in our setting, our bounds are same up to constants to analogous bounds for logistic linear regression in \cite{shamir2021gradient,pmlr-v178-schliserman22a}. As discussed in Sec. \ref{sec:properties} this is possible thanks to the \sbwc property.

The condition $m\gtrsim\|w-w_0\|^4$ on the norm of the weights controls the maximum deviations of weights $w$ from initialization (with respect to network width) required for our results to guarantee arbitrarily small train loss. Specifically, to get the most out of Theorem 
\ref{thm:train} we need to choose appropriate $w$ that satisfies both the condition $m\gtrsim\|w-w_0\|^4$ and keeps the associated ridge-regularized loss $\widehat{F}(w)+\|w-w_0\|^2/(\eta T)$ small. This combined requirement is formalized in the neural-net realizability Assumption \ref{ass:real} below. As we will discuss later in Section \ref{sec:ntk}, this assumption translates into an assumption on the underlying data distribution that ultimately enables the application of Theorem \ref{thm:train} to achieve vanishing training error.

\begin{ass}[NN--Realizability]\label{ass:real}
There exists a decreasing function $g:\R_+\rightarrow\R_+$ which measures the 
norm of deviations from initialization of models that achieve arbitrarily small training error.

Formally, for almost surely all $n$ training samples and for any sufficiently small $\eps>0$ there exists $w^{(\eps)}\in\R^{d'}$ such that 
\bea\nn
\hf(w^{(\eps)})\le\eps,\;\;\;\text{and}\;\;\; g(\eps)=\left\|w^{(\eps)}-w_0\right\|.
\eea
\end{ass}

Since Assumption \ref{ass:real} holds for arbitrarily small $\eps$, it guarantees that the model has enough capacity to interpolate the data, i.e., attain train error that is arbitrarily small ($\eps$). Additionally, this is accomplished for model weights whose distance from initialization is managed by the function $g(\eps)$. By using these model weights  to select $w$ in Theorem \ref{thm:train} we obtain train loss bounds for interpolating models.
\begin{theorem}
[Training loss under interpolation]\label{thm:train_IP}
Let Assumptions \ref{ass:features}-\ref{ass:real} hold. Let $\eta\le\min\{\frac{1}{L_\hf},g(1)^2,\frac{g(1)^2}{\hf(w_0)}\}$ and assume the width satisfies $m\ge 18^2 L^2 R^4 \, g(\frac{1}{T})^4$ for a fixed training horizon $T$. Then,
\begin{align}
        &\hf(w_{T}) \;\leq \frac{2}{T}+ 
        \frac{5\, g(\frac{1}{T})^2}{2\eta T},
        \label{eq:tr-ip}
        \\
        &\forall t\in[T] \;:\;\;\big\|w_t-w_0\big\|\;\le\; 4 \, g(\frac{1}{T}).\nn
\end{align}
\end{theorem}

To interpret the theorem's conclusions suppose that the function $g(\cdot)$ of Assumption \ref{ass:real} is at most logarithmic; i.e., $g(\frac{1}{T})=O (\log (T))$. Then, Theorem \ref{thm:train_IP} implies that  $m=\Omega(\log^4(T))$ neurons suffice to achieve train loss $\tilde O(\frac{1}{T})$ while GD iterates at all iterations satisfy $\|w_t-w_0\|=O(\log (T))$. In Section \ref{sec:ntk} (see also Remark \ref{rem:linear}), we will give examples of data separability conditions that guarantee the desired logarithmic growth of $g(\cdot)$ for logistic loss minimization, which in turn imply the favorable convergence guarantees described above. Under the same conditions we will show that the step-size requirement simplifies to $\eta\le\min\{3,1/L_\hf\}$ (see  Corollary \ref{cor:ntkres}). Finally, we remark that Theorem \ref{thm:train_IP} provides sufficient parameterization conditions under which GD  with $T=\tilde{\Omega}(n)$ iterations finds  weights $w_T$ that yield an interpolating classifier and thus, achieve zero training error.
To see this, assume logistic loss and observe setting $T\gtrsim n$ in Eq. \eqref{eq:tr-ip} gives $\hf (w_T)\le \log (2) /n$. This in turn implies that every sample loss satisfies $\hf_i(w_T)\le \log(2)$, equivalently $y_i=\operatorname{sign}\big(\Phi(w_T,x_i)\big)$.

\subsection{Generalization}\label{sec:gen gap}

Our main result below bounds the generalization gap of GD for training two-layer nets with self-bounded loss functions. We remark that all expectations that appear below are over the training set.

\begin{theorem}[Generalization gap -- General bound]\label{thm:generalization gap general}
Suppose Assumptions \ref{ass:features}-\ref{ass:self boundedness} hold. Fix any time horizon $T\geq 1$ and any step size $\eta\leq 1/L_\Fh$ where  $L_\hf$ is the objective's smoothness parameter. Let any $w\in\R^\dpr$ such that $\|w-w_0\|^2\geq \max\{\eta T\, \Fh(w),\eta \Fh(w_0)\}.$ Suppose  hidden-layer width $m$ satisfies 
$
m\geq 64^2 L^2 R^4 \|w-w_0\|^4.
$
Then, the generalization gap of GD at iteration $T$ is bounded as
\[
    \E\Big[F(w_T)-\Fh(w_T)\Big] \leq \frac{8\ell^2 R^2}{n}\, \E\left[ \eta T \,\Fh(w) + {2\|w-w_0\|^2}\right].
\]

\end{theorem}
A few remarks regarding the theorem are in place. The theorem's assumptions are similar to those in Theorem \ref{thm:train}, which bounds the training loss. The condition $\|w-w_0\|^2\geq \max\{\eta T \Fh(w),\eta \Fh(w_0)\}$ needs to hold almost surely over the training data, which is non-restrictive, as in later applications of the theorem, the choice of $w$ arises from Assumption \ref{ass:real}. The condition $m\geq 64^2L^2R^4\|w-w_0\|^4$ on the width of the network, is also the same as that of Theorem \ref{thm:train} but with a larger constant. This means that the last-iterate train loss bound from Theorem \ref{thm:train} (Eq. \eqref{eq:train loss bound final}) holds under the setting of Theorem \ref{thm:generalization gap general}. Hence, it applies to the expected train loss $\E[\Fh(w_T)]$ and, combined with the generalization-gap bound, yields a bound on the expected test loss $\E[F(w_T)]$. 

To optimize the bound, a proper $w$ must be selected by minimizing the population version of a ridge-regularized training objective. In interpolation settings, the procedure for selecting $w$ follows the same guidelines as in Assumption \ref{ass:real} and in a similar style as obtaining Theorem \ref{thm:train_IP}. 

\begin{theorem}[Generalization gap under interpolation]\label{thm:gengap}
                Let Assumptions \ref{ass:features}-\ref{ass:real} hold. Fix $T\ge 1$ and let $m\ge 64^2 L^2R^4\, g(\frac{1}{T})^4$. Then, for any $\eta\le\min\{\frac{1}{L_\hf},g(1)^2,\frac{g(1)^2}{\hf(w_0)}\}$ the expected generalization gap at iteration $T$ satisfies 
    \bea
  \E \Big[{F}(w_T)-\widehat F(w_T)\Big]\le \frac{24\ell^2 R^2 \,g(\frac{1}{T})^2 }{n} \,.
    \eea
    \end{theorem}
Note the width condition is similar in order to that of Theorem \ref{thm:train_IP}. 
Thus, provided $g(\frac{1}{T})\lesssim \log(T)$ (see Remark \ref{rem:linear} and Section \ref{sec:ntk} for examples), 
we have generalization gap of order $\tilde O(\frac{1}{n})$ with $m=\Omega(\log^4 (T))$ neurons. Combined with the training loss guarantees from Theorem \ref{thm:train_IP}, we have test loss rate $\tilde O(\frac{1}{T}+\frac{1}{n})$. This further implies that with $m\approx \log^4(n)$ neurons and $T=n$ iterations, the test loss reaches the optimal rate of $\tilde O(\frac{1}{n})$. On the other hand, previous stability-based generalization bounds (e.g., \cite{richards2021learning}) required polynomial width  $m\gtrsim T^2$ and eventually obtained sub-optimal generalization rates of order $O(\frac{T}{n})$. We further discuss the technical novelties resulting in these improvements in Section \ref{sec:ps}.

\begin{remark}[Example: Linearly-separable data]\label{rem:linear}
    Consider logistic-loss minimization,  $\tanh$ activation  $\sigma(u)=\frac{e^u-e^{-u}}{e^u+e^{-u}}$ and data distribution that is linearly separable with margin $\gamma$, i.e., for almost surely all $n$ samples there exists  unit-norm vector $v^\star\in\R^d$ such that $\forall i\in[n]: y_i \langle v^\star,x_i\rangle \ge \gamma$. We  initialize the weights to zero, i.e. $w_0=0$ and show that  the realizability Assumption \ref{ass:real} naturally holds in this setting. To see this, for any fixed $\eps>0,$ set $\alpha=\frac{2\left(\log(1/\eps)\right)}{\gamma\sqrt{m}}$ and assume $m\geq 4\log^2(1/\eps)$. With this choice, 
    select weights $w^{(\eps)}_j:= \alpha v^\star, a_j=\frac{1}{\sqrt{m}}$ for $j\in[1,\cdots,\frac{m}{2}]$ and $w^{(\eps)}_j:= -\alpha v^\star, a_j=\frac{-1}{\sqrt{m}}$ for $j\in\{\frac{m}{2}+1,\cdots,m\}$. Then, the model output for any sample $(x_i,y_i)$ satisfies
    \[
    y_i\Phi(w^{(\eps)},x_i) = \frac{y_i\sqrt{m}}{2}\left(\sigma(\alpha \inp{v^\star}{x_i})-\sigma(-\alpha \inp{v^\star}{x_i})\right) = y_i \sqrt{m} \sigma(\alpha \inp{v^\star}{x_i}) \geq \sqrt{m} \sigma(\alpha \gamma) \geq \frac{\sqrt{m}}{2} \alpha \gamma = \log(\frac{1}{\eps})
    \]
    where the first equality uses the fact that $tanh$ is odd, the first inequality follows by the increasing nature of $tanh$ and data separability, and the last inequality follows since $\alpha\gamma\leq 1$ and $\sigma(u)\geq u/2$ for all $u\in[0,1].$ Thus, the loss satisfies $\Fh(w^{(\eps)})\leq \eps$ since for the logistic function $\log(1+e^u)\leq e^u.$ Moreover,  our choice of $\alpha$ implies $g(\eps)=\|w^{(\eps)}-w_0\|=\|w^{(\eps)}\|=\alpha\sqrt{m}=2\log(1/\eps)/\gamma.$ To conclude, the NN-Realizability Assumption \ref{ass:real} holds with $g(\eps)=2\log(1/\eps)/\gamma$ and thus applying Theorems \ref{thm:train_IP}, \ref{thm:gengap} shows that with $m=\Omega(\log^4(T))$ neurons, the training loss and generalization gap are bounded by $\tilde O(\frac{1}{\gamma^2 T})$ and $\tilde O(\frac{1}{\gamma^2 n})$, respectively. We note that the same conclusion as above holds for other smooth activations such as Softmax or GELU.
    
\end{remark}

\section{On Realizability of NTK-Separable Data}\label{sec:ntk}

In this section, we interpret our results for NTK-separable data by showing that our realizability condition holds for this class. We recall the definition of NTK-separability below \cite{nitanda2019gradient,chen2020much,cao2020generalization}. 
\begin{ass}[Separability by NTK]\label{ass:ntksep}
For almost surely all $n$ training samples from the data distribution there exists $w^\star\in\R^{d'}$ and $\gamma>0$ such that $\|w^\star\|=1$ and for all $i\in[n]$,
\bea
y_i \Big\langle \nabla_1\Phi(w_0,x_i),w^\star\Big\rangle \ge \gamma.
\eea
\end{ass}
We also assume a bound on the model's output at initialization. Similar assumptions, but for the value of the loss, also appear in prior works that study generalization using the algorithmic stability framework \cite{richards2021stability,leistability}. 
\begin{ass}[Initialization bound]\label{ass:init}
   There exists parameter $C$ such that $\forall i\in[n] :|\Phi(w_0,x_i)|\le C,$ for almost surely all $n$ training samples from the data distribution
\end{ass}
The next proposition relates the NTK-separability assumption to our realizability assumption. The proofs for this section are given in Appendix \ref{sec:ntkapp}.
\begin{proposition}[Realizability of NTK-separable data] \label{propo:ntksepreal}
 Let Assumptions \ref{ass:features}-\ref{ass:act},\ref{ass:ntksep}-\ref{ass:init} hold. Assume $f(\cdot)$ to be the logistic loss. Fix $\eps>0$ and let $m\ge \frac{L^2 R^4}{4\gamma^4 C^2} (2C+\log(1/\eps))^4$. Then the realizability Assumption \ref{ass:real} holds with $g(\eps)=\frac{1}{\gamma}(2C+\log(1/\eps))$. In other words, there exists $w^{(\eps)}$ such that
 \bea
\hf(w^{(\eps)})\le \eps, \;\;\;\text{and}\;\;\; \left\|w^{(\eps)}-w_0\right\| = \frac{1}{\gamma}\left(2C+\log(1/\eps)\right).
 \eea
\end{proposition}

Having established realizability, the following is an immediate corollary of the general results presented in the last section.

\begin{corollary}[Results under NTK-separability]\label{cor:ntkres}
    Let Assumptions \ref{ass:features}-\ref{ass:act},\ref{ass:ntksep}-\ref{ass:init} hold and assume logistic loss. 
    Suppose $m\ge \frac{64^2L^2R^4}{\gamma^4}(2C+\log(T))^4$ for a fixed training horizon $T$. Then for any $\eta\le\min\{3,\frac{1}{L_\hf}\}$, the training loss and generalization gap are bounded as follows:
    \begin{align}
    &\hf(w_T) \le \frac{5(2C+\log(T))^2}{\gamma^2 \eta T},\nn\\
    &\E\left[F(w_T)-\hf(w_T)\right] \le \frac{24\ell^2 R^2}{\gamma^2 n}(2C+\log(T))^2.\nn
    \end{align}
\end{corollary}

A few remarks are in place regarding the corollary. By Corollary \ref{cor:ntkres}, we can conclude that the expected generalization rate of GD on logistic loss and NTK-separable data as per Assumption \ref{ass:ntksep} is $\tilde{O}(\frac{1}{n})$ provided width $m=\Omega(\log^4(T))$. Moreover, the expected training loss is $\E[\hf(w_T)] = \toc(\frac{1}{T})$. Thus, the expected test loss after $T$ steps is $\toc(\frac{1}{T}+\frac{1}{ n})$. In particular for $T=\Omega(n)$, the expected test loss becomes $\tilde{O}(\frac{1}{n})$. This rate is optimal with respect to sample size and only requires  polylogarithmic hidden width with respect to $n$, specifically, $m=\Omega(\log^{4}(n))$. Notably, it represents an improvement over prior stability results, e.g., \cite{richards2021learning} which required polynomial width and yielded suboptimal generalization rates of order $O(T/n)$. It is worth noting that the test loss bound's dependence on the margin, particularly the $\frac{1}{\gamma^2 n}$-rate obtained in our analysis, bears similarity to the corresponding results in the convex setting of linearly separable data recently established in \cite{shamir2021gradient,pmlr-v178-schliserman22a}.
Additionally, our results improve upon corresponding bounds for neural networks obtained via Rademacher complexity analysis \cite{Ji2020Polylogarithmic,chen2020much} which yield generalization rates $\tilde O(\frac{1}{\sqrt{n}})$. Moreover, these works have a $\gamma^{-8}$ dependence on margin for the minimum network width, whereas in Corollary \ref{cor:ntkres} this is reduced to $\gamma^{-4}$. We also note that in general, both $\gamma$ and $C$ may depend on the data distribution, the data dimension, or the nature of initialization. This is demonstrated in the next section where we apply the corollary above to the noisy XOR data distribution and Gaussian initialization.
 
\begin{remark}[Benefits of exponential tail]\label{rem:polyn}
      We have stated Corollary \ref{cor:ntkres} for the logistic loss, which has an exponential tail behavior. For general self-bounded loss functions and by following the same steps, we can show a bound on generalization gap of order $O(\frac{1}{n} (f^{-1}(\frac{1}{T}))^{2})$ provided $m=\Omega((f^{-1}(\frac{1}{T}))^4)$. Hence, the tail behavior of $f$ controls both the generalization gap and minimum width requirement. In particular, under Assumption \ref{ass:ntksep}, polynomial losses with tail behavior $f(u)\sim1/{u^\beta}$ result in generalization gap $O({T^{2/\beta}}/{n})$ for $m=\Omega(T^{4/\beta})$. Thus, increasing the rate of decay $\beta$ for the loss, improves both bounds on generalization and width. This suggests the benefits of self-bounded fast-decaying losses such as exponentially-tailed loss functions for which the dependence on $T$ is indeed only logarithmic. 
\end{remark}

\subsection*{Example: Noisy XOR data}
Next, we specialize the results of the last section to the noisy XOR data distribution  \cite{wei2019regularization} and derive the corresponding margin and test-loss bounds. Consider the following $2^d$ points,
\bea\nn
x_i = (x_i^1,x_i^2,\cdots,x_i^d) \in\{(1,0),(0,1),(-1,0) ,(0,-1)\} \times \{-1,1\}^{d-2},
\eea
where $\times$ denotes the Cartesian product and the labels are determined as $y_i=-1$ if $x_i^1=0$ and $y_i=1$ if $x_i^1=\pm 1$. Moreover, consider normalization $\overline{x}_i= \frac{1}{\sqrt{d-1}} x_i$ so that $R=  1.$ The noisy XOR data distribution is the uniform distribution over the set with elements $(\overline{x}_i,y_i)$. For this dataset and Gaussian initialization, \cite{Ji2020Polylogarithmic} have shown for ReLU activation that the NTK-separability assumption holds with margin $\gamma=\Omega(1/d)$. In the next result, we compute the margin for activation functions that are convex, Lipshitz and locally strongly convex. 
\begin{proposition}[Margin]\label{propo:marginntksep}
    Consider the noisy XOR data $(\overline{ x}_i,y_i)\in\R^d\times\{\pm 1\}$. Assume the activation function is convex, $\ell$-Lipschitz and $\mu$-strongly convex in the interval $[-2,2]$ for some $\mu>0$, i.e., $\min_{t\in[-2,2]} \sigma''(t) \ge \mu$. Moreover, assume Gaussian initialization $w_0\in\R^{d'}$ with entries iid $N(0,1)$. If $m \ge \frac{80^2 d^3 \ell^2}{2\mu^2} \log (2/\delta)$, then with probability at least $1-\delta$ over the initialization, the NTK-separability Assumption \ref{ass:ntksep} is satisfied with margin $\gamma = \frac{\mu}{80 d}$.  
\end{proposition}

An interesting example of an activation function that satisfies the mentioned assumptions is the Softplus activation where $\sigma(u)=\log (1+e^{u})$. This activation function has $\mu=0.1$ and $\ell=1$, and it is also smooth with $L=1/4$. Therefore, the results on generalization and training loss presented in Corollary \ref{cor:ntkres} hold for it.  For noisy XOR data, Proposition \ref{propo:marginntksep} shows the margin in Assumption \ref{ass:ntksep} is $\gamma\gtrsim 1/d$. Additionally, for standard Gaussian initialization we have by Lemma \ref{lem:init} that with high-probability the initialization bound in Assumption \ref{ass:init} satisfies $C \lesssim \sqrt{d}$. Putting these together, and applying Corollary \ref{cor:ntkres} shows that GD with $n$ training samples reaches test loss rate $\tilde O(\frac{d^3}{n})$ after $T\approx n$ iterations and given $m=\tilde\Omega(d^6)$ neurons. It is worth noting that the number of training samples can be exponentially large with respect to $d$. In this case the minimum width requirement is only polylogarithmic in $n$.

\section{Proof Sketches}
\label{sec:ps}
We discuss here high-level proof ideas for both optimization and generalization bounds of Theorems \ref{thm:train} and \ref{thm:generalization gap general}. Formal proofs are deferred to Appendices \ref{sec:training analysis} and \ref{sec:gen analysis}. 

\subsection{Training loss}\label{sec:optimization sketch}

As already discussed in Section \ref{sec:properties}, the key insight we use to obtain bounds that are analogous to results for optimizing convex objectives, is to exploit the \sbwc property of the objective in Eq. \eqref{eq:self-bounded hessian lamin}. Thanks to this property, the Hessian minimum eigenvalue $\lambda_\min(\nabla^2\Fh(w_t))$ becomes less negative at the same rate at which the train loss $\Fh(w_t)$ decreases.

The technical challenge at formalizing this intuition arises as follows. Controlling the rate at which $\Fh(w_t)$ converges to $\Fh(w)$ for the theorem's $w$ requires controlling the Hessian at \emph{all} intermediate points $w_{\alpha t}:=\alpha w_t+(1-\alpha)w, \alpha\in[0,1]$ between $w$ and GD iterates $w_t$. This is due to Taylor's theorem used to relate $\Fh(w_t)$ to the target value $\Fh(w)$ as follows:
\begin{align*}
    \widehat F(w) &\geq \widehat F(w_t) + \left\langle\nabla \hf(w_t),w-w_t\right\rangle + \frac{1}{2}\,\lamin{\nabla^2 \hf(w_{\alpha t})} \Big\|w-w_t\Big\|^2.
\end{align*}
Thus from self-bounded weak convexity, to control the last term above we need to control $\Fh(w_{\alpha t})$ for any intermediate point $w_{\alpha t}$ along the GD trajectory. This is made possible by establishing the following  generalized local quasi-convexity property.

\begin{proposition}[Generalized Local Quasi-Convexity]\label{propo:GLQC}
Suppose $\hf:\R^{d'}\rightarrow\R$ satisfies the \sbwc property in Eq. \eqref{eq:self-bounded hessian lamin} with parameter $\kappa$. Let $w_1,w_2\in\R^{d'}$ be two arbitrary points with  distance $\left\|w_1-w_2\right\|\leq D<\sqrt{2/\kappa}$ . Set $\tau:=\left(1-\kappa D^2/2\right)^{-1}$. Then,
\bea\label{eq:glqc}
\max_{v\in[w_1,w_2]} \widehat F(v)\le \tau\cdot \max\{\widehat F(w_1),\widehat F(w_2)\}.
\eea
\end{proposition}
Recall that quasi-convex functions satisfy Eq. \eqref{eq:glqc} with $\tau=1$ and $D$ can be unboundedly large. The Proposition \ref{propo:GLQC} indicates that our neural-net objective function is approximately quasi-convex (since $\tau>1$) and this property holds locally, i.e. provided that  $w_1,w_2$ are sufficiently close.

Applying \eqref{eq:glqc} for $w_1=w_t,w_2=w$ allows controlling $\Fh(w_{\alpha t})$ in terms of the train loss $\Fh(w_t)$ and the target loss $\Fh(w)$. The only additional requirement in Proposition \ref{propo:GLQC} for this to hold is that
\begin{align}\label{eq:m large roughly}
 1/\kappa\propto \sqrt{m} \gtrsim \|w_t-w\|^2.
\end{align}
This condition exactly determines the required neural-net width. Formally, we have the following.
\begin{corollary}[GLQC of sufficiently wide neural nets]\label{cor:GLQC}
Let Assumptions \ref{ass:features},\ref{ass:act}, \ref{ass:self boundedness} hold. Fix arbitrary $w_1,w_2\in\R^{\dpr}$, any constant $\lambda >1$, and  $m$ large enough such that ${\sqrt{m}} \ge \lambda \frac{LR^2}{2}\|w_1-w_2\|^2$.
Then, 
\bea\label{eq:GLQC-nn}
\max_{v\in[w_1,w_2]} \widehat F(v)\le \left(1-1/\lambda\right)^{-1}\cdot \max\{\widehat F(w_1),\widehat F(w_2)\}.
\eea
\end{corollary}

To conclude, using Corollary \ref{cor:GLQC}, we can show the regret bound in Eq. \eqref{eq:train loss bound final} provided (by \eqref{eq:m large roughly}) that $\sqrt{m}\gtrsim \|w_t-w\|^2$ is true for all $t\in[T].$ To make the width requirement independent of $w_t$, we then use a recursive argument to prove that $\|w_t-w\|\leq 3 \|w-w_0\|$. These things put together, lead to the parameter bound $\|w_t-w_0\|\leq 4 \|w-w_0\|$ and the width requirement $\sqrt{m}\gtrsim \|w-w_0\|^2$ in the theorem's statement. We note that the GLQC property is also crucially required for the generalization analysis which we discuss next.


\subsection{Generalization gap}\label{sec:sketch-gen}
We bound the generalization gap using stability analysis \cite{bousquet2002stability,hardt2016train}. In particular, we use \cite[Thm. 2]{lei2020fine} that relates the generalization gap to the ``on average model stability''. Formally, let $w_t^\negi$ denote the $t$-th iteration of GD on the leave-one-out loss $\hf^\negi(w):=\frac{1}{n}\sum_{j\neq i} \hf_j(w)$. As before, $w_t$ denotes the GD output on full-batch loss $\Fh$. We will use the fact (see Corollary \ref{cor:gradients_and_hessians_bounds}) that $f(y\Phi(\cdot,x))$ is $G_{\Fh}$-Lipschitz with $G_{\Fh}=\ell R$ under Assumptions \ref{ass:act} and \ref{ass:loss lipschitz}. Then, using \cite[Thm. 2(a)]{lei2020fine} (cf. Lemma \ref{lem:sk22}) it holds that
\begin{align}\label{eq:stability to gen_gap}
       \E\Big[{F}(w_T)-\widehat F(w_T)\Big] \leq  2 G_{\Fh} \;\E\Big[\frac{1}{n}\sum_{i=1}^n\|w_T-w_T^{\neg i}\|\Big].
   \end{align}
In order to bound the on-average model-stability term on the right-hand side above we need to control the degree of expansiveness of GD. Recall that for convex objectives GD is non-expansive (e.g. \cite{hardt2016train}), that is $\|\big(w-\eta\nabla \hf(w)\big)-\big(\wpr-\eta\nabla \hf(\wpr)\big)\|\leq \|w-\wpr\|$ for any $w,\wpr$. For the non-convex objective in our setting, the lemma below establishes a generalized non-expansiveness property  via leveraging the structure of the objective's Hessian for the two-layer net.
\begin{lemma}[GD-Expansiveness]\label{lem:expansiveness}
Let Assumptions \ref{ass:features} and \ref{ass:act} hold. For any $w,\wpr\in\R^\dpr$, any step-size $\eta>0$, and $\wa:=\alpha w+(1-\alpha) \wpr$ it holds for $H(w):=\eta \frac{LR^2}{\sqrt{m}}\widehat{F}^{\prime}(w) + \max\left\{1,\eta \ell^2R^2 \widehat{F}^{\prime \prime}(w)\right\}$ that
\begin{align*}
    &\Big\|\Big(w-\eta\nabla \widehat F(w)\Big)-\Big(\wpr-\eta\nabla \widehat F(\wpr)\Big)\Big\| \leq \max_{\alpha\in[0,1]} H(w_\alpha)\,\left\|w-\wpr\right\|,
\end{align*}
 where we define $\hf'(w):=\frac{1}{n}\sum_{i=1}^n |f'(y_i\Phi(w,x_i))|$ and $\hf''(w):=\frac{1}{n}\sum_{i=1}^n f''(y_i\Phi(w,x_i))$.
\end{lemma}
This lemma can be further simplified for the class  of self-bounded loss functions. Specifically, using $|\fp(u)|\le f(u)$ and $\fpp(u)\le 1$ from Assumptions \ref{ass:self boundedness} and \ref{ass:loss smooth}, we immediately deduce the following. 
\begin{corollary}[Expansiveness for self-bounded losses]\label{cor:9.1}
    In the setting of Lemma \ref{lem:expansiveness}, further assume the loss satisfies Assumptions \ref{ass:loss smooth} and \ref{ass:self boundedness}. 
    Provided $\eta\le{1}/{(\ell^2R^2)}$, it holds for all $w,w'\in\R^\dpr$ that
\begin{align}\label{eq:tr-ne}
    \Big\|\left(w-\eta\nabla \widehat F(w)\right)-\left(\wpr-\eta\nabla \widehat F(\wpr)\right)\Big\| \leq \Big(1+\eta \frac{LR^2}{\sqrt{m}}\max_{\alpha\in[0,1]} \widehat F(\wa) \Big)\,\Big\|w-\wpr\Big\|\,.
\end{align}
\end{corollary}
In Eq. \eqref{eq:tr-ne} the expansiveness is weaker than in a convex scenario, where the coefficient would be $1$ instead of $1+\frac{\eta LR^2}{\sqrt{m}}\max_{\alpha\in[0,1]} \hf(\wa)$. However, for self-bounded losses (i.e. $\abs{\fp(u)}\leq f(u)$) the ``gap to convexity'' $\frac{\eta LR^2}{\sqrt{m}}\max_{\alpha\in[0,1]} \hf(\wa)$ in Corollary \ref{cor:9.1}  is better than the  gap  from Lemma \ref{lem:expansiveness} for 1-Lipschitz losses (i.e. $\abs{\fp(u)}\leq 1$), which would be $\frac{\eta LR^2}{\sqrt{m}}$. Indeed, after unrolling the GD iterates, the latter eventually leads to polynomial width requirements \cite{richards2021learning}.

Instead, to obtain a polylogarithmic width, we use the expansiveness bound in Eq. \eqref{eq:tr-ne} for self-bounded losses together with the \GLQC property in Corollary \ref{cor:GLQC} as follows. From Corollary \ref{cor:GLQC},  if $m$ is large enough such that 
\bea\nn
\sqrt{m} \geq LR^2 \|w_t-w_t^\negi\|^2, \;\;\;\;\;\;\forall t\in[T],\;\, \forall i\in[n],
\eea
then Eq. \eqref{eq:GLQC-nn} holds on the GD path.
This further simplifies the result of Corollary \ref{cor:9.1} applied for $w=w_t, \wpr=w_t^\negi$ into
\[
  \Big\|\big(w_t-\eta\nabla \widehat F^\negi(w_t)\big)-\big(w_t^\negi-\eta\nabla \widehat F^\negi(w_t^\negi)\big)\Big\| \leq \widetilde{H}^i_t\;\Big\|w_t-w_t^\negi\Big\|\,,
\]
where $\widetilde{H}^i_t:=1+  \frac{2\eta L R^2}{\sqrt{m}} \max\{\Fh^\negi(w_t), \Fh^\negi(w_t^\negi)
  \}.$
Now from the optimization analyses in Sec. \ref{sec:optimization sketch}, we know intuitively that $\Fh^\negi(w_t)\leq\Fh(w_t)$ decays at rate $\tilde{O}(1/t)$; thus, so does $\Fh^\negi(w_t^\negi)$. Therefore, for all $i\in[n]$ the expansivity coefficient $\widetilde{H}^i_t$ in the above display is decaying to $1$ as GD progresses. 

To formalize all these and connect them to the model-stability term in \eqref{eq:stability to gen_gap}, note using triangle inequality and the Gradient Self-boundedness property of Lemma \ref{lem:key prop intro} that
\[
\Big\|w_t-w_t^\negi\Big\|\leq \Big\|\big(w_t-\eta\nabla \widehat F^\negi(w_t)\big)-\big(w_t^\negi-\eta\nabla \widehat F^\negi(w_t^\negi)\big)\Big\|+ \frac{\eta \ell R}{n}\Fh_i(w_t)\,.
\]
Unrolling this display over $t\in[T]$, averaging over $i\in[n]$, and using our expansiveness bound above we show in Appendix \ref{sec:gen analysis} the following bound for the model stability term
\begin{align}\label{eq:sketch stability eq}
    \frac{1}{n}\sum_{i=1}^n\left\|w_T-w_T^\negi\right\| \leq \frac{\eta\ell R e^\beta}{n}\sum_{t=0}^{T-1}\widehat{F}(w_t)\,,
\end{align}
where $\beta\lesssim \left(\sum_{t=1}^T\Fh(w_t)+\sum_{t=1}^T\Fh^\negi(w_t^\negi)\right)\big/\sqrt{m}\,.$ But, we know from training-loss bounds in Theorem\,\ref{thm:train} that  $\sum_{t=1}^T\Fh(w_t)\lesssim \|w-w_0\|^2$ (and similar for $\sum_{t=1}^T\Fh^\negi(w_t^\negi)$). Thus, $\beta\lesssim \|w-w_0\|^2\big/\sqrt{m}$. At this point, the theorem's conditions guarantees $\sqrt{m}\gtrsim \|w-w_0\|^2$, so that $\beta=O(1)$. Plugging back in \eqref{eq:sketch stability eq} we conclude with the following stability bound:
 $\frac{1}{n}\sum_{i=1}^n\|w_T-w_T^\negi\|\lesssim \sum_{t=0}^{T}\widehat{F}(w_t)\big/n.
$
Applying the train-loss bounds of Theorem \ref{thm:train} once more  completes the proof.

\section{Prior Works}
The theoretical study of generalization properties of neural networks (NN) is more than two decades old \cite{NIPS1996_fb2fcd53,bartlett-VCdim}. Recently, there has been  an increased interest in understanding and improving generalization of SGD/GD on over-parameterized neural networks, e.g.  \cite{allen2019learning, oymak2020toward,javanmard2020analysis,richards2021learning}. These results however typically require very large width where $m=\text{poly}(n)$. We discuss most-closely related-works below.

\noindent\textbf{Quadratic loss.} For quadratic loss, \cite{li2018learning,soltanolkotabi2018theoretical,allen2019convergence,oymak2020toward,liu2022loss} showed that sufficiently over-parameterized neural networks of polynomial width satisfy a local Polyak-Łojasiewicz (PL) condition $\|\nabla \widehat F(w)\|^2 \ge 2\mu (\widehat F(w)-\widehat F^\star)$, where $\mu$ is at least the smallest eigenvalue of the neural tangent kernel matrix. The PL property in this case implies that the training loss converges linearly with the rate $\widehat F(w_t) = O((1-\eta\mu)^t)$ if the GD iterates remain in the PL region. Moreover,  \cite{charles2018stability,lei2020sharper},  have used the PL condition to further characterize stability properties of corresponding non-convex models. Notably, \cite{lei2020sharper} derived order-optimal rates $O(\frac{1}{\mu n})$ for the generalization loss. However these rates only apply to quadratic loss. Models trained with logistic or exponential loss on separable data do \emph{not} satisfy the PL condition even for simple interpolating linear models. Aside from the PL condition-related results, but again for quadratic loss, \cite{oymak2019generalization} showed under specific assumptions on the data translating to low-rank NTK, that logarithmic width is sufficient to obtain classification error of order $O(n^{-1/4})$. In general, they achieve error rate $O(n^{-1/2})$, but for $m=\tilde{\Omega}(n^2)$. 

\noindent\textbf{Logistic-loss minimization with linear models.} Logistic-loss minimization is more appropriate for classification and rate-optimal generalization bounds for GD have been obtained recently in the linear setting, where the training objective is convex.
In particular, for linear logistic regression on data that are linearly separable with margin $\gamma>0$, \cite{shamir2021gradient} proved a finite-time test-error bound $O(\frac{\log^2 T}{\gamma^2 T}+\frac{\log^2 T}{\gamma^2 n})$. Ignoring $\log$ factors, this is order-optimal with the sample size $n$ and training horizon $T.$ Their proof uses exponential-decaying properties of the logistic loss to control the norm of gradient iterates, which it cleverly combines with Markov's inequality to bound the fraction of well-separated datapoints at any iteration. This in turn translates to a test-error bound by  standard margin-based generalization bounds. More recently, \cite{pmlr-v178-schliserman22a} used algorithmic-stability analysis proving same rates (up to log factors) for the test loss. Their results hold for general convex, smooth, self-bounded and decreasing objectives under a realizability assumption suited for convex objectives (analogous to Assumption \ref{ass:real}). Specifically, this includes linear logistic regression with linearly separable data. Here, we show that analogous rates on the test loss hold true for more complicated nonconvex settings where data are separable by shallow neural networks.

\noindent\textbf{Stability of GD in NN.} State-of-the-art generalization bounds on shallow neural networks via the stability-analysis framework have appeared very recently in \cite{richards2021learning,richards2021stability,leistability}. For Lipschitz losses, \cite{richards2021learning} shows that the empirical risk is weakly convex with a weak-convexity parameter that improves as the neural-network width $m$ increases. Leveraging this observation, they establish stability bounds for GD iterates at time $T$ provided sufficient parameterization $m=\tilde\Omega(T^2)$. Since the logistic loss is Lipschitz, these bounds also apply to our setting. Nevertheless, our work improves upon \cite{richards2021learning} in that: (i) we require significantly smaller width, poly-logarithmic rather than polynomial, and (ii) we show $\tilde{O}(1/n)$ test loss bounds in the realizable setting, while their bounds are $O(T/n).$ Central to our improvements is a largely refined analysis of the curvature of the loss via identifying and proving a generalized quasi-convexity property for neural networks of polylogarithmic width trained with self-bounded losses (see Section \ref{sec:ps} for details). Our results also improve upon the other two works \cite{richards2021stability,leistability}, which both require polynomial widths. However, we note that these results are not directly comparable since \cite{richards2021stability,leistability} focus on  quadratic-loss minimization. See also Appendix \ref{sec:more related}.

\noindent\textbf{Uniform convergence in NN.} Uniform bounds on the generalization loss have been derived in literature via Rademacher complexity analysis \cite{bartlett2002rademacher}; see for example \cite{neyshabur2015norm,arora2019fine,golowich2020size,vardi2022the,frei2022random} for a few results in this direction. These works typically obtain the bounds of order  $O(\frac{\mathcal{R}}{\sqrt{n}})$, where $\mathcal{R}$ depends on the Rademacher complexity of the hypothesis space. Recent works by \cite{Ji2020Polylogarithmic,chen2020much} also utilized Rademacher complexity analysis to obtain test loss rates of $O(1/\sqrt{n})$ under an NTK separability assumption (see also \cite{nitanda2019gradient}) with polylogarithmic width requirement for shallow and deep networks, respectively. Instead, while maintaining minimal width requirements, we obtain test-loss rates $\tilde{O}(1/n)$, which are order-optimal.
Our approach, which is based on algorithmic-stability, is also different and uncovers new properties of the optimization landscape, including a generalized local quasi-convexity property. On the other hand, the analysis of \cite{Ji2020Polylogarithmic,chen2020much} applies to ReLU activation and bounds the test loss with high-probability over the sampling of the training set. Instead, we require smooth activations similar to other studies such as \cite{oymak2019generalization,chatterji2021does,baibeyond,nitanda2019gradient,richards2021learning,richards2021stability,leistability} and we bound the test loss in expectation over the training set. Finally, we also note that data-specific generalization bounds for two-layer nets have also appeared recently in \cite{cao2022benign,frei2022implicit}. However, those results require that data are nearly-orthogonal.

\noindent\textbf{Convergence/implicit bias of GD.}  Convergence and implicit bias of GD for logistic/exponential loss functions on linear models and neural networks have been investigated in \cite{ji2018risk,soudry2018implicit,nacson2019convergence,lyu2020gradient,chizat2020implicit,chatterji2021does}. In particular, \cite{lyu2020gradient,ji2020directional} have shown for homogeneous neural-networks that GD converges in direction to a max-margin solution. While certainly powerful, this implicit-bias convergence characterization becomes relevant only when the number $T$ of GD iterations is exponentially large. Instead, our convergence bounds apply for finite $T$ (on the order of sample size), thus are more practically relevant. Moreover, their results assume a GD iterate $t_0$ such that $\hf(w_{t_0})\le \log(2)/n$. Similar assumption appears in \cite{chatterji2021does}, which require initialization $\hf(w_0)\le 1/n^{1+C}$ for constant $C>0$. Our approach is entirely different:
we prove that sufficient parameterization benefits the loss curvature and suffices for GD steps to find an interpolating model and attain near-zero training loss, provided data satisfy an appropriate realizability condition.


\section{Conclusions}
In this paper we study smooth shallow neural networks trained with self-bounded loss functions, such as logistic loss. Under interpolation, we provide minimal sufficient parameterization conditions to achieve rate-optimal generalization and optimization bounds. These bounds improve upon prior results which require substantially large over-parameterization or obtain sub-optimal generalization rates. Specifically, we significantly improve previous stability-based analyses in terms of both relaxing the parameterization requirements and obtaining improved rates. Although our focus was on binary classification with shallow networks, our approach can be extended to multi-class settings and deep networks, which will be explored in future studies. Extending our results to the stochastic case by analyzing SGD is another important future direction. Moreover, while our current treatment relies on smoothness of the activation function to exploit properties of the curvature of the training objective, we aim to examine the potential of our results to extend to non-smooth activations. Finally, our generalization analysis bounds the expectation of the test loss (over data sampling) and it is an important future direction extending these guarantees to a  high-probability setting.

\bibliographystyle{apalike}
\bibliography{main}
\clearpage
\appendix
\counterwithin{theorem}{section}
\section{Training Loss Analysis}\label{sec:training analysis}

This section includes the proofs of the results stated in Section \ref{sec:training loss}.

\subsection{Proof of Theorem \ref{thm:train}}
We begin with proving the general train-loss and parameter-norm bounds of Theorem \ref{thm:train}. In fact, we state and prove a slightly more general statement of the theorem which includes non-smooth and non-Lipschitz losses (such as expoential loss) that satisfy a second order self-bounded property described below. 

\begin{ass}[2nd order self-boundedness]\label{ass:2nd order self-boundedness}
    The convex loss function $f:\R\rightarrow\R_+$ satisfies the 2nd order self-boundedness property, i.e. 
    \[
    \fpp(u)\leq f(u), \forall u\in\R.
    \]
\end{ass}
\begin{theorem}[General statement of Theorem \ref{thm:train}] \label{thm:train-general}  Let Assumptions \ref{ass:features}-\ref{ass:act} hold. Assume the loss function satisfies self-bounded Assumption \ref{ass:self boundedness}. Moreover, suppose either Assumption \ref{ass:lip and smooth} or Assumption  \ref{ass:2nd order self-boundedness}  hold. Fix any $T\geq 0$. Let the step-size satisfy the assumptions of the descent lemma (Lemma \ref{lem:descentlm}). Assume any $w$ and $m$ such that
 $\|w-w_0\|^2\geq \max\left\{\eta T \widehat F(w),\eta \widehat F(w_0)\right\}$ and $m\geq 18^2 L^2R^4\|w-w_0\|^4$. Then, the training loss and the parameters' norm satisfy
\begin{align}
        &\frac{1}{T}\sum_{t=1}^T \widehat F(w_t)  \leq 2 \widehat F(w) + \frac{5\|w-w_0\|^2}{2\eta T},\label{eq:thm13}\\
        &\forall t\in[T]\;\; :\;\; \|w_t-w_0\|\le 4\|w-w_0\|.\nn
\end{align}
\end{theorem}

To prove Theorem \ref{thm:train-general}, we first state our descent lemma for both self-bounded losses and lipschitz-smooth losses.

\begin{lemma}[Descent lemma]\label{lem:descentlm}
Let Assumptions \ref{ass:features}-\ref{ass:act} hold. Assume the loss function satisfies self-boundedness Assumptions \ref{ass:self boundedness},\ref{ass:2nd order self-boundedness}. Then, for any $\eta< \frac{1}{R^2\,\widehat F(w_t)}\min\{\frac{1}{{\ell^2 + L}}, \frac{1}{\sqrt{L}\ell} \}$ the descent property holds, i.e., \[
    \widehat F(w_{t+1})\leq \widehat F(w_t) - \frac{\eta}{2}\|\nabla \widehat F(w_t)\|^2.
    \]
Moreover, if $f$ satisfies Assumption \ref{ass:lip and smooth} then the descent property holds for any $\eta\le 1/L_\hf$ where $L_\hf:=\ell^2 R^2 + \frac{L R^2}{{\sqrt{m}}}$ is the smoothness parameter of the training objective.
\end{lemma}

\begin{proof}
Due to self-boundedness Assumption \ref{ass:2nd order self-boundedness}, as well as Assumptions \ref{ass:features}-\ref{ass:act} the objective  is also self-bounded according to Corollary \ref{cor:gradients_and_hessians_bounds}, i.e., 
$\|\nabla^2 \widehat F(w)\| \leq \left(\ell^2 R^2 + \frac{L R^2}{{\sqrt{m}}}\right) \widehat F(w), \|\nabla \widehat F(w)\| \leq \ell R \,\widehat F(w)$.

By Taylor's expansion, there exists a $w'\in[w_t,w_{t+1}]$ such that,
 \begin{align*}
 \widehat F(w_{t+1}) &=  \widehat F(w_t) + \left\langle \nabla \widehat F(w_t),w_{t+1}-w_t\right\rangle + \frac{1}{2} \left\langle w_{t+1}-w_t,\nabla^2 \widehat F(w')\,(w_{t+1}-w_t)\right\rangle\\
&\le  \widehat F(w_t) + \left\langle \nabla \widehat F(w_t),w_{t+1}-w_t \right\rangle + \frac{1}{2}\max_{v\in[w_t,w_{t+1}]} \left\|\nabla ^2 \widehat F(v)\right\|\cdot\left\|w_{t+1}-w_t\right\|^2   \\[4pt]
&\le \widehat F(w_t) -\eta \|\nabla \widehat F(w_t)\|^2 + \frac{\eta^2 {\left(\ell^2 R^2 + \frac{L R^2}{{\sqrt{m}}}\right)}}{2}\max_{v\in[w_t,w_{t+1}]} \widehat F(v)\cdot \left\|\nabla \widehat F(w_t)\right\|^2.
\end{align*}
By Corollary \ref{lem:GLQC}, for $\sqrt{m}\ge \eta^2 L\ell^2 R^4 \hf^2(w_t)\ge L R^2 \|\eta\nabla \hf(w_t)\|^2=LR^2\|w_{t+1}-w_t\|^2$ it holds that $$\max_{v\in[w_t,w_{t+1}]} \widehat F(v) \le 2 \max\{\hf(w_t),\hf(w_{t+1})\},$$
which yields
\bea\label{eq:dsct}
\hf (w_{t+1})\le \widehat F(w_t) -\eta \|\nabla \widehat F(w_t)\|^2 + \eta^2 {\left(\ell^2 R^2 + \frac{L R^2}{{\sqrt{m}}}\right)}  \max\left\{\widehat F(w_t),\widehat F(w_{t+1})\right\}\cdot \|\nabla \widehat F(w_t)\|^2.
\eea
We note that the condition on $m$ simplifies to $m\ge 1$ if $\eta \le\frac{1}{\sqrt{L} \ell R^2}\frac{1}{\hf(w_t)}$. 
\par
Back to \eqref{eq:dsct},
if $\widehat F(w_{t+1})\ge \widehat F(w_t)$ by our condition $\eta< \frac{1}{{\ell^2 R^2 + L R^2 /\sqrt{m}}}\frac{1}{\widehat F(w_t)}$ it holds that
\begin{align*}
\widehat F(w_{t+1}) &\le \widehat F(w_t) + \eta \|\nabla \widehat F(w_t)\|^2 \left( \frac{\widehat F(w_{t+1})}{\widehat F(w_t)}-1\right)\\[4pt]
&\le \widehat F(w_t) + \eta\ell^2R^2 \hf^2(w_t)\left(\frac{\widehat F(w_{t+1})}{\widehat F(w_t)}-1\right).
\end{align*}
Since $\eta< \frac{1}{\ell^2 R^2}\frac{1}{\widehat F(w_t)}$,
\bea
\widehat F(w_{t+1}) &< \widehat F(w_t) + \widehat F(w_t)  \left( \frac{\widehat F(w_{t+1})}{\widehat F(w_t)}-1\right)\nn\\
&= \widehat F(w_{t+1})\,,\nn
\eea
which is a contradiction. Thus it holds that $\widehat F(w_{t+1})< \widehat F(w_t)$. Continuing from Eq. \eqref{eq:dsct} with the assumption $\eta< \frac{1}{{\ell^2 R^2 +  L R^2/ \sqrt{m}}}\frac{1}{\widehat F(w_t)}$, we conclude that
\bea
\widehat F(w_{t+1})&\le \widehat F(w_t) - \eta \|\nabla \widehat F(w_t)\|^2 + \frac{1}{2}\eta^2 {\left(\ell^2 R^2 + \frac{L R^2}{{\sqrt{m}}}\right)}  \widehat F(w_t)\cdot \|\nabla \widehat F(w_t)\|^2\nn\\[4pt]
&\le \widehat F(w_t) - \frac{\eta}{2} \|\nabla \widehat F(w_t)\|^2.\nn
\eea
This completes the proof for self-bounded losses.
\par
Next, suppose $f$ is $1$-smooth and $1$-Lipschitz.  Then, as per Corollary \ref{cor:gradients_and_hessians_bounds}, $\hf$ is smooth with the constant $$L_\hf:=\ell^2 R^2 + \frac{L R^2}{{\sqrt{m}}}.$$ 
Following similar steps as in the beginning of proof and assuming step-size $\eta\leq 1/L_\hf$ we immediately conclude that,
\begin{align}
    \widehat F(w_{t+1})&\leq \widehat F(w_t) - \eta\|\nabla \widehat F(w_t)\|^2 + \frac{\eta^2 L_\hf}{2}\|\nabla \widehat F(w_t)\|^2\nn\\
    &\le \widehat F(w_t)- \frac{\eta}{2}\|\nabla \widehat F(w_t)\|^2.\nn
\end{align}    
This completes the proof.
\end{proof}

As a remark, the descent property implies that the loss decreases by each step, i.e., $\widehat F(w_t)\le \widehat F(w_0)$. Thus for self-bounded losses the condition $\eta< \frac{1}{R^2\,\widehat F(w_0)}\min\{\frac{1}{{\ell^2 + L}}, \frac{1}{\sqrt{L}\ell} \}$ is sufficient. We also note that the Lipschitz-smoothness and 2nd order self-bounded assumptions are only required for the descent lemma above, which results in conditions on the step-size based on the properties of loss. In the rest of the proof we only use the self-bounded Assumption \ref{ass:self boundedness} in order to use the self-bounded weak convexity property of the objective (see Def. \ref{defn:sbhs}).   
\par
Next lemma finds a general relation for the training loss in terms of an arbitrary point $w\in\R^{d'}$ and the fluctuations of loss between $w$ and GD iterates $w_t$. 
\begin{lemma}\label{lem:train_loss_general}
Let Assumptions \ref{ass:features}-\ref{ass:act} hold. Assume the loss function satisfies the self-bounded Assumption \ref{ass:self boundedness}. Moreover, suppose $\hf$ and step-size $\eta$ are such that  the following descent condition is satisfied for all $t\ge 0$:
\begin{align}\label{eq:standard_descent}
    \widehat F(w_{t+1})\leq \widehat F(w_t) - \frac{\eta}{2}\|\nabla \widehat F(w_t)\|^2.
\end{align}
Then, for any $w\in\R^{d'}$ it holds that
    \begin{align}
    \frac{1}{T}\sum_{t=1}^T \widehat F(w_t) \leq \widehat F(w) + \frac{\|w-w_0\|^2}{\eta T} + \frac{1}{2}\frac{LR^2}{{\sqrt{m}}} \frac{1}{T}\sum_{t=0}^{T-1}\max_{\alpha\in[0,1]}\widehat F(w_{\alpha t}) \,\|w-w_t\|^2,\nn
\end{align}
where we set $w_{\alpha t}:=\alpha w_t+(1-\alpha) w.$

\end{lemma}
\begin{proof}

Fix any $w$. By Taylor, there exists $w_{\alpha t}, \alpha\in[0,1]$ such that
\begin{align*}
    \widehat F(w) &= \widehat F(w_t) +\left\langle\nabla \widehat F(w_t),w-w_t\right\rangle + \frac{1}{2}\left\langle w-w_t,\nabla^2 \widehat F(w_{\alpha t})\, (w-w_t)\right\rangle
    \\
    &\geq \widehat F(w_t) + \left\langle\nabla \widehat F(w_t),w-w_t\right\rangle + \frac{1}{2}\lamin{\nabla^2 \widehat F(w_{\alpha t})} \|w-w_t\|^2
    \\
    &\geq \widehat F(w_t) + \left\langle\nabla \widehat F(w_t),w-w_t\right\rangle - \frac{1}{2}\frac{LR^2}{{\sqrt{m}}} \widehat F(w_{\alpha t}) \,\|w-w_t\|^2.
\end{align*}
The last line is true by Corollary \ref{cor:gradients_and_hessians_bounds}. Thus, for any $w$,
\begin{align*}
    \widehat F(w)\geq \widehat F(w_t) + \left\langle \nabla \widehat F(w_t),w-w_t\right\rangle - \frac{1}{2}\frac{LR^2}{{\sqrt{m}}} \max_{\alpha\in[0,1]} \widehat F(w_{\alpha t}) \,\|w-w_t\|^2.
\end{align*}
Plugging this in \eqref{eq:standard_descent} gives
\begin{align}
    \widehat F(w_{t+1})&\leq \widehat F(w) - \left\langle\nabla \widehat F(w_t),w-w_t\right\rangle - \frac{\eta}{2}\left\|\nabla \widehat F(w_t)\right\|^2+\frac{1}{2}\frac{LR^2}{{\sqrt{m}}} \max_{\alpha\in[0,1]}\widehat F(w_{\alpha t}) \,\|w-w_t\|^2\nn
    \\
    &= \widehat F(w) +\frac{1}{\eta}\left(\|w-w_t\|^2-\|w-w_{t+1}\|^2 \right) +\frac{1}{2}\frac{LR^2}{{\sqrt{m}}} \max_{\alpha\in[0,1]}\widehat F(w_{\alpha t}) \,\|w-w_t\|^2.\label{eq:descent with weak convexity}
\end{align}
where the second line follows by completion of squares using $w_{t+1}-w_t=-\eta\nabla \widehat F(w_t)$.

Telescoping the above display for $t=0,\ldots,{T-1}$, we arrive at the desired.
\end{proof}

Next, when $m$ is large enough so that we can invoke the \GLQC property, the bound of Lemma \ref{lem:train_loss_general} takes the following convenient form

\begin{lemma}\label{lem:train_loss_final}    Let the assumptions of Lemma \ref{lem:train_loss_general} hold. Assume $w$ and $m$ such that ${\sqrt{m}} \ge  2LR^2\|w-w_t\|^2$ for all $t\in[T-1]$ then
    \bea\label{eq:train_loss_final_lemma_eqn}
        \frac{1}{T}\sum_{t=1}^T \widehat F(w_t)  \leq 2 \widehat F(w) + \frac{2\|w-w_0\|^2}{\eta T} +  \frac{\widehat F(w_0)}{2T} .
    \eea
    \end{lemma}
    \begin{proof}
    We invoke Corollary \ref{lem:GLQC} with $\lambda=4$ to deduce that for all $t\in[T-1]$
    \bea
    \max_{\alpha\in[0,1]} \widehat F(w_{\alpha t}) &\le \frac{4}{3}\max \{ \widehat F(w),\widehat F(w_t)\} < \frac{4}{3} \widehat F(w_t) + \frac{4}{3} \widehat F(w).\label{eq:implication of quasiconvexity}
    \eea
    Noting the assumption on $m$ and recalling Lemma \ref{lem:train_loss_general},
    \begin{align*}
        \frac{1}{T}\sum_{t=1}^T \widehat F(w_t) &\leq \widehat F(w) + \frac{\|w-w_0\|^2}{\eta T} + \frac{1}{2}\frac{LR^2}{{\sqrt{m}}} \frac{1}{T}\sum_{t=0}^{T-1}\max_{\alpha\in[0,1]}\widehat F(w_{\alpha t}) \,\|w-w_t\|^2\\
        &\le \frac{4}{3}\widehat F(w) + \frac{\|w-w_0\|^2}{\eta T} + \frac{1}{3T}\sum_{t=0}^{T-1} \widehat F(w_t)\\
        &\le \frac{4}{3}\widehat F(w) + \frac{\|w-w_0\|^2}{\eta T} + \frac{1}{3T}\sum_{t=0}^{T} \widehat F(w_t).
    \end{align*}
    Arranging terms yields the desired result. 
    \end{proof}



Finally, using the about bounds on the training loss, we can bound the parameter-norm using a recursive argument presented in the lemma below.

\begin{lemma}[Iterates-norm bound]\label{lem:w bound general exponential}
   Suppose the assumptions of Lemma \ref{lem:train_loss_general} hold. Fix any $T\geq 0$
and assume any $w$ and $m$ such that
 \begin{align}\label{eq:w requirement constant}
\|w-w_0\|^2\geq \max\{\eta T \widehat F(w),\eta \widehat F(w_0)\}.
\end{align}
 and
\begin{align}\label{eq:m_large_general_constant}
{\sqrt{m}} \geq 18 LR^2\|w-w_0\|^2,
\end{align}
Then, for all $t\in[T]$,
    \begin{align}\label{eq:wt minus w bound no square}
    \|w_t-w\| \leq 3\| w-w_0\|.
    \end{align}
\end{lemma}
\begin{proof}
Denote $A_t=\|w_t-w\|$.
    Start by recalling from \eqref{eq:descent with weak convexity} that for all $t$:
    \begin{align}
        A_{t+1}^2\leq A_t^2 +  \eta \widehat F(w) - \eta \widehat F(w_{t+1})  +\eta\,\frac{LR^2}{2{\sqrt{m}}} \max_{\alpha\in[0,1]}\widehat F(w_{\alpha t}) \,A_t^2.\label{eq:tighter bound start}
    \end{align}

We will prove the desired statement \eqref{eq:wt minus w bound no square} using induction.  
For $t=0$, $A_0=\|w-w_0\|$. Thus, the assumption of induction holds. Now assume \eqref{eq:wt minus w bound no square} is correct for $t\in[T-1]$, i.e. $A_t\leq 3 \|w-w_0\|, \forall t\in[T-1]$. We will then prove it holds for $t=T$.  

The first observation is that by induction hypothesis
{${\sqrt{m}}\geq  18 LR^2 \|w-w_0\|^2\geq 2LR^2 A_t^2$} for all $t\in[T-1].$ Thus, for all $t\in[T-1]$, the condition of the generalized local quasi-convexity Corollary \ref{cor:GLQC} holds for $\lambda=4$ implying (see also \eqref{eq:implication of quasiconvexity})
\[
\forall t\in[T-1]\,:\,\,\max_{\alpha\in[0,1]} \widehat F(w_{\alpha t}) \le  \frac{4}{3} \widehat F(w_t) + \frac{4}{3} \widehat F(w).
\]
Using this in \eqref{eq:tighter bound start} we find for all $t\in[T-1]$ that
\begin{align*}
A_{t+1}^2&\leq A_t^2 +  \eta \widehat F(w) - \eta \widehat F(w_{t+1})  +\eta\,\frac{LR^2\cdot A_t^2}{2{\sqrt{m}}}\left(\frac{4}{3} \widehat F(w_t) + \frac{4}{3} \widehat F(w)\right)
\\
&\leq A_t^2 +  \eta \widehat F(w) - \eta \widehat F(w_{t+1})  +\eta\,\left(\frac{1}{3} \widehat F(w_t) + \frac{1}{3} \widehat F(w)\right) 
\end{align*}
where in the second inequality we used again that ${\sqrt{m}}\ge 2LR^2 A_t^2.$
We proceed by telescoping the above display over $t=0,1,\ldots,T-1$ to get
\begin{align*}
    A_T^2&\leq A_0^2 + \frac{4}{3}\eta T \Fh(w) + \frac{1}{3}\eta \Fh(w_0) + \frac{1}{3}\eta \sum_{t=0}^{T-1}\Fh(w_t) - \eta \Fh(w_T)
    \\&\leq A_0^2 + \frac{4}{3}\eta T \Fh(w) + \frac{2}{3}\eta \Fh(w_0) + \frac{1}{3}\eta  \sum_{t=1}^{T}\Fh(w_t), 
\end{align*}
where the second line follows by nonegativity of the loss.

Now, to bound the last term above, observe that the condition of Lemma \ref{lem:train_loss_final} holds since $\sqrt{m}\geq 2LR^2A_t^2$ for all $t\in[T-1]$ by induction hypothesis. Hence, using \eqref{eq:train_loss_final_lemma_eqn}, we conclude that
\begin{align}
    A_T^2&\leq A_0^2 + \frac{4}{3}\eta T \Fh(w) + \frac{2}{3}\eta \Fh(w_0) + \frac{1}{3}\eta T \left(2 \widehat F(w) + \frac{2A_0^2}{\eta T} +  \frac{\widehat F(w_0)}{2T}\right) \nn
    \\
    &= \frac{5}{3}A_0^2 + 2\eta T \Fh(w) + \frac{5}{6}\eta \Fh(w_0) 
    \nn\\
    &\leq \frac{5}{3}\|w-w_0\|^2 + 2\|w-w_0\|^2 + \frac{5}{6}\|w-w_0\|^2 =\frac{9}{2}\|w-w_0\|^2 
    \qquad \implies \,A_T \leq 3 \|w-w_0\|.
\end{align}
In the last inequality, we used the assumptions of the lemma on $\|w-w_0\|$ and $A_0= \|w-w_0\|$. This completes the proof.
\end{proof}

\noindent\textbf{Completing the proof of Theorem \ref{thm:train-general}.}\par
The proof follows from combining the bounds on the training loss and parameters' growth from Lemmas \ref{lem:train_loss_final}-\ref{lem:w bound general exponential} and noting that with condition on $\|w-w_0\|^2$ from Lemma \ref{lem:w bound general exponential} we have $\hf(w_0)\le \|w-w_0\|^2/\eta$ to derive \eqref{eq:thm13}. Moreover, we have $\|w_t-w_0\|\le \|w_t-w\|+\|w-w_0\|\le 4\|w-w_0\|$. 

\subsection{Proof of Theorem \ref{thm:train_IP}}

Here we prove training loss bound for interpolating NN as asserted by Theorem \ref{thm:train_IP}. Similar to the previous section, we prove a more general result where the loss is not necessarily Lipschitz or smooth.
We are now ready to prove Theorem \ref{thm:train_IP} for general self-bounded losses. In particular, Theorem \ref{thm:train_IP} follows directly from the next result by choosing $f$ to be Lipschitz and smooth. 
\begin{theorem}[General statement of Theorem \ref{thm:train_IP}]\label{thm:train_IPapp}
Suppose Assumptions \ref{ass:features}-\ref{ass:act}, \ref{ass:self boundedness} hold. Moreover, assume the objective and data satisfy the Assumption \ref{ass:real}. Let the step-size satisfy the assumptions of Descent Lemma \ref{lem:descentlm}. Moreover, assume $\eta\le\min\{g(1)^2,\frac{1}{L_{\Fh}},\frac{g(1)^2}{\widehat F(w_0)}\}$ and $m\ge 18^2 L^2 R^4 \, g(\frac{1}{T})^4$ for a fixed training horizon $T$. Then,
\begin{align*}
            &\hf(w_{T}) \;\leq \frac{2}{T}+ 
        \frac{5\, g(\frac{1}{T})^2}{2\eta T},
        \\
        &\forall t\in[T] \;:\;\;\big\|w_t-w_0\big\|\;\le\; 4 \, g(\frac{1}{T}).\nn
\end{align*}
\end{theorem}
\begin{proof}
According to Assumption  \ref{ass:real}, for any sufficiently small $\eps>0$, there exists a $w^{(\eps)}$ such that $\widehat F(w^{(\eps)})\le\eps$ and $\|w^{(\eps)}-w_0\| = g(\eps)$. Pick $\eps= 1/T$. With the condition $\eta\le\min\{g(1)^2,g(1)^2/\widehat F(w_0)\}$ we have 
$$\max\left\{\eta T \widehat F(w^{(1/T)}),\eta \widehat F(w_0)\right\}\le g(1)^2\le g(\frac{1}{T})^2=\|w^{(1/T)}-w_0\|^2,$$
where in the second inequality we used the fact that $g$ is a decreasing function. The desired result is obtained by Theorem \ref{thm:train-general}.
\end{proof}

\subsection{Generalized local quasi-convexity property}
In the remainder of this section, we prove the generalized local quasi-convexity property.
\begin{proposition}[Restatement of Proposition \ref{propo:GLQC}]\label{prop:GLQCapp}
Suppose $\hf:\R^{d'}\rightarrow\R$ satisfies the \sbwc property in Eq. \ref{eq:self-bounded hessian lamin} with parameter $\kappa$. Let $w_1,w_2\in\R^{d'}$ be two arbitrary points with  distance $\left\|w_1-w_2\right\|\leq D<\sqrt{2/\kappa}$ . Set $\tau:=\left(1-\kappa D^2/2\right)^{-1}$. Then,
\bea
\max_{v\in[w_1,w_2]} \widehat F(v)\le \tau\cdot \max\{\widehat F(w_1),\widehat F(w_2)\}.
\eea
\end{proposition}
\begin{proof}
Assume the claim of the proposition is incorrect, then \bea
\max_{v\in[w_1,w_2]} \widehat F(v)> \tau\cdot \max\{\widehat F(w_1),\widehat F(w_2)\}> \max\{\widehat F(w_1),\widehat F(w_2)\}.\label{eq:GLQCp}
\eea

Define $w_\star :=\arg\max_{v\in[w_1,w_2]} \widehat F(v)$. Note that $w_\star$ is an interior point. Thus by the optimality condition it holds
\bea
\left\langle\nabla \widehat F(w_\star), w_1-w_2\right\rangle = 0.
\eea

By Taylor's approximation theorem for two points $w_1,w\in\R^{d'}$, there exists a $w_\beta\in[w,w_1]$, such that
\bea
\widehat F(w_1) &= \widehat F(w) + \left \langle \nabla \widehat F(w),w_1-w\right\rangle + \frac{1}{2}\left\langle w-w_1, \nabla^2 \widehat F(w_\beta) \,(w-w_1)\right\rangle\label{eq:glqctaylor}
\eea
Pick $w=w_\star = \alpha_\star w_1+(1-\alpha_\star) w_2$ in Eq. \eqref{eq:glqctaylor}, and note that
$$\left \langle \nabla \widehat F(w_\star),w_1-w_\star\right\rangle = -(1-\alpha_\star)\left \langle \nabla \widehat F(w_\star),w_1-w_2\right\rangle=0.$$
Therefore,
\bea\nn
\widehat F(w_1) &= \widehat F(w_\star) + \frac{1}{2}\left\langle w_\star-w_1, \nabla^2 \widehat F(w_\beta) \,(w_\star-w_1)\right\rangle\\
&\ge \widehat F(w_\star) +\frac{1}{2} \lambda_{\min}(\nabla^2 \widehat F(w_\beta))\, \Big\|w_\star-w_1\Big\|^2\nn \\
&\ge \widehat F(w_\star) -\frac{1}{2} \kappa\,\widehat F(w_\beta) \Big\|w_\star-w_1\Big\|^2.\nn
\eea
where in the last line we used the self-bounded weak convexity property i.e., $\lambda_\min\left(\nabla^2 \widehat F(w_\beta)\right) \geq - \kappa\widehat F(w_\beta)$.

This leads to
\begin{align*}
\widehat F(w_1) &\ge \widehat F(w_\star) -\frac{(1-\alpha_\star)^2}{2} \kappa\, \widehat F(w_\beta) \Big\|w_1-w_2\Big\|^2\\
&> \widehat F(w_\star) -\frac{1}{2} \kappa\, \widehat F(w_\beta) \Big\|w_1-w_2\Big\|^2.
\end{align*}
Note that $w_\beta\in[w_\star,w_1]\subset[w_1,w_2]$, thus $\widehat F(w_\beta)\le \widehat F(w_\star)$ by definition of $w_\star.$ Therefore, 
\begin{align*}
\widehat F(w_\star) &< \frac{1}{1-\frac{1}{2}\kappa\,\|w_1-w_2\|^2} \widehat F(w_1)\\
&\le\frac{1}{1-\frac{1}{2}\kappa D^2} \widehat F(w_1),
\end{align*}
which is in contradiction with \eqref{eq:GLQCp}. This proves the statement of the proposition.
\end{proof}

Specializing this property to two-layer neural networks yields the following. 
\begin{corollary}[Restatement of Corollary \ref{cor:GLQC}]\label{lem:GLQC}
Let Assumptions \ref{ass:features},\ref{ass:act}, \ref{ass:self boundedness} hold. Fix arbitrary $w_1,w_2\in\R^{\dpr}$, any constant $\lambda >1$, and  $m$ large enough such that ${\sqrt{m}} \ge \lambda \frac{LR^2}{2}\|w_1-w_2\|^2$.
Then, 
\bea
\max_{v\in[w_1,w_2]} \widehat F(v)\le \left(1-1/\lambda\right)^{-1}\cdot \max\{\widehat F(w_1),\widehat F(w_2)\}.
\eea
\end{corollary}
\begin{proof}
 By our assumptions and Corollary \ref{cor:gradients_and_hessians_bounds} the objective's Hessian satisfies $$\lambda_\min\left(\nabla^2 \widehat F(w)\right) \geq - \frac{LR^2}{{\sqrt{m}}}\widehat F(w).$$
 Invoking Proposition \ref{prop:GLQCapp} with $\kappa:= \frac{LR^2}{{\sqrt{m}}}$ concludes the claim.
\end{proof}
\section{Generalization analysis}
\label{sec:gen analysis}
This section includes the proofs of the generalization results stated in Section \ref{sec:gen gap}.
\subsection{Proof of Theorem \ref{thm:generalization gap general}}
We prove the generalization gap of Theorem \ref{thm:generalization gap general} for Lipshitz-smooth losses. The proof follows the steps of our proof sketch in Sec. \ref{sec:sketch-gen}. 
\par
First, the proofs of exansiveness of GD in NN (Lemma \ref{lem:expansiveness}) and the corresponding model stability bound are given next. 
\begin{lemma}[GD-Expansivieness]
Let Assumptions \ref{ass:features}-\ref{ass:act} hold. For any $w,\wpr$ and $\wa=\alpha w+(1-\alpha) \wpr$ it holds that
\begin{align*}
    &\left\|\Big(w-\eta\nabla \widehat F(w)\Big)-\Big(\wpr-\eta\nabla \widehat F(\wpr)\Big)\right\| \leq \max_{\alpha\in[0,1]} H(w_\alpha)\,\left\|w-\wpr\right\|,\\
    &H(w):=\eta \frac{LR^2}{\sqrt{m}}\widehat{F}^{\prime}(w) + \max\left\{1,\eta \ell^2R^2 \widehat{F}^{\prime \prime}(w)\right\},
\end{align*}
 where we define $\hf'(w):=\frac{1}{n}\sum_{i=1}^n |f'(y_i\Phi(w,x_1))|$ and $\hf''(w):=\frac{1}{n}\sum_{i=1}^n f''(y_i\Phi(w,x_1))$.
\end{lemma}
\begin{proof}
Fix $u\,:\,\|u\|=1$ and define $g_u:\R^{d'}\rightarrow\R$:
$$
g_u(w):= \left\langle u,w\right\rangle  - \eta\langle u,\nabla \widehat F(w)\rangle.
$$
Note 
$$
\left\|w-\nabla \widehat F(w) - (w'-\nabla \widehat F(w'))\right\| = \max_{\|u\|=1} \left|g_u(w)-g_u(w')\right|.
$$

For any $w,w'$, we have
\begin{align}
    g_u(w)-g_u(w')&=\int_0^1 u^\top\left(I-\eta\nabla^2\widehat F(w'+\alpha(w-w'))\right)(w-w')\mathrm{d}\alpha\nn
    \\
    &\leq \max_{\alpha\in[0,1]}\left\|\left(I-\eta\nabla^2\widehat F(w'+\alpha(w-w'))\right)\right\|\Big\|w-w'\Big\|.
\end{align}
For convenience denote $w_\alpha:=\alpha w + (1-\alpha)w'$ and $A_\alpha:=\nabla^2\widehat F(w_\alpha)$. Then, for any $\alpha\in[0,1]$ we have that
\begin{align}\label{eq:max sing to max eig}
\left\|I-\eta\nabla^2\widehat F(w_\alpha)\right\| = \max\Big\{\abs{1-\eta\lambda_\min(A_\alpha)},\abs{1-\eta\lambda_\max(A_\alpha)}\Big\}.
\end{align}

For convenience, let $\beta:=\frac{1}{\sqrt{m}}LR^2 \widehat{F}^{\prime}(\wa)\geq 0$ and note from Lemma \ref{lem:gradients_and_hessians_bounds_general} that $\lambda_\min(A_\alpha)\geq -\beta.$ Using this, we will show that
\begin{align}\label{eq:lamin expansive}
\abs{1-\eta\lambda_\min(A_\alpha)}\leq \max\{1+\eta \beta,\eta\lambda_\max(A_\alpha)\}.
\end{align}
To show this consider two cases. First, if $\eta\lambda_\min(A_\alpha)\in[-\eta \beta,1]$, then $$
\abs{1-\eta\lambda_\min(A_\alpha)}=1-\eta\lambda_\min(A_\alpha)\leq 1+\eta \beta.
$$
On the other hand, if $\eta\lambda_\min(A_\alpha)\geq 1$, then
$$
\abs{1-\eta\lambda_\min(A_\alpha)}=\eta\lambda_\min(A_\alpha)-1\leq \eta\lambda_\min(A_\alpha) \leq \eta\lambda_\max(A_\alpha),
$$
which shows \eqref{eq:lamin expansive}.

Next, we will show that 
\begin{align}\label{eq:lamax expansive}
\abs{1-\eta\lambda_\max(A_\alpha)}\leq \max\{1+\eta \beta,\eta\lambda_\max(A_\alpha)\}.
\end{align}
We consider again three cases. First, if $\eta\lambda_\max(A_\alpha)\in[0,1]$, then 
$$
\abs{1-\eta\lambda_\max(A_\alpha)} =1-\eta\lambda_\max(A_\alpha) \leq 1.
$$
Second, if $\eta\lambda_{\max}(A_\alpha)\geq1$
$$
\abs{1-\eta\lambda_\max(A_\alpha)}=\eta\lambda_\max(A_\alpha)-1\leq \eta\lambda_\max(A_\alpha).
$$
Otherwise, it must be that $-\beta\leq\lambda_{\min}(A_\alpha)\leq\lambda_{\max}(A_\alpha)\leq 0$. Thus,
\[
\abs{1-\eta\lambda_\max(A_\alpha)}=1-\eta\lambda_\max(A_\alpha)\leq 1-\eta\lambda_\min(A_\alpha)\leq 1+\eta\beta.
\]

To complete the proof of the lemma combine \eqref{eq:max sing to max eig} with \eqref{eq:lamin expansive} and \eqref{eq:lamax expansive}:
$$
\|I-\eta\nabla^2\widehat F(w_\alpha)\| \leq \max\{1+\eta\beta,\eta\lambda_\max(A_\alpha)\},
$$
and further use from Lemma \ref{lem:gradients_and_hessians_bounds_general} that $\eta\lambda_\max(A_\alpha)\leq \eta\ell^2 R^2\widehat{F}^{\prime \prime}(w) + \eta\beta.$
\end{proof}



For the stability analysis below, recall the definition of the leave-one-out (loo) training loss for $i\in[n]$:
$\Fh^\negi(w):=\frac{1}{n}\sum_{j\neq i} \Fh_j(w)$. With these, define the loo model updates of GD on the loo loss:
\[
w_{t+1}^\negi:=w_{t}^\negi -\eta \nabla \Fh^\negi(w_t^\negi),~t\geq 0,\qquad w_0^\negi=w_0.
\]

\begin{theorem}[Model stability bound]\label{thm:model_stability}
Suppose Assumptions \ref{ass:features}, \ref{ass:act}, \ref{ass:lip and smooth}, \ref{ass:self boundedness} hold. Fix any time horizon $T\geq 1$ and any step size $\eta>0$. 
Set the regret and the leave-one-out regrets of GD updates as follows:
\[
\regret:=\frac{1}{T}\sum_{t=1}^T\Fh(w_t)\qquad\text{and}\qquad \regret_\loo:=\frac{1}{T}\max_{i\in[n]} \sum_{t=1}^T\Fh^\negi(w_t^\negi).
\]
Suppose that the width $m$ is large enough so that it satisfies the following two conditions:
\begin{align}
\sqrt{m}\geq 4LR^2\max\left\{\|w_t-w_0\|^2,\|w_t^\negi-w_0\|^2\right\} \,, \quad\forall i\in[n], t\in[T]\,,\label{eq:m large model stability 1}
\end{align}
and
\begin{align}
    \sqrt{m}\geq 6LR^2\eta T\max\left\{\regret,\regret_\loo\right\}\,.
    \label{eq:m large model stability 2}
\end{align}
Then, the leave-one-out model stability is bounded as follows:
   \begin{align*}
       \frac{1}{n}\sum_{i=1}^n\Big\|w_T-w_T^{\neg i}\Big\| \leq \frac{2\eta \ell R}{n}  \left(\Fh(w_0) + T\cdot\regret \right).
   \end{align*}
\end{theorem}
\begin{proof}
Using self-boundedness Assumption \ref{ass:self boundedness} together with Corollary \ref{cor:9.1} it holds for all $i\in[n]$: 
    \begin{align}
    \Big\|w_{t+1}-w_{t+1}^{\neg i}\Big\| &\leq \Big\| \left(w_t-\eta\nabla \widehat{F}^\negi(w_t)\right)-\left(w_t^{\neg i}-\eta\nabla \widehat{F}^\negi(w_t^{\neg i})\right)\Big\| + \frac{\eta}{n} \Big\|\nabla \hf_i(w_t)\Big\|\nn
    \\
    &\leq \Big\| \left(w_t-\eta\nabla \widehat{F}^\negi(w_t)\right)-\left(w_t^{\neg i}-\eta\nabla \widehat{F}^\negi(w_t^{\neg i})\right)\Big\| + \frac{\eta\ell R}{n} \hf_i(w_t)\nn
    \\
    &\leq \left(1+\eta \frac{LR^2}{\sqrt{m}}\max_{\alpha\in[0,1]}\widehat{F}^{\neg i}(w^{\neg i}_{\alpha t}) \right)\,\Big\|w_t-w_t^{\neg i}\Big\| + \frac{\eta\ell R}{n} \hf_i(w_t),\label{eq:wt - wti}
    \end{align}
    where we denote for convenience $w^{\neg i}_{\alpha t} = \alpha w_t + (1-\alpha)w_t^{\neg i}.$

Moreover, by the theorem's condition in Eq. \eqref{eq:m large model stability 1}, it holds for all $t\in[T]$ and all $i\in[n]$ that 
\bea
\sqrt{m} 
\ge 2 LR^2 (\|w_t-w_0\|^2+\|w_t^{\neg i}-w_0\|^2)\geq LR^2\left\|w_t-w_t^{\neg i}\right\|^2.\nn
\eea
Thus,  we can apply Corollary \ref{cor:GLQC} for $\lambda=2$, which gives the following \GLQC property for the loo objective:
\bea
\max_{\alpha\in[0,1]}\widehat F^\negi(w_{\alpha t}^{\neg i}) \le 2 \max\left\{\widehat{F}^{\neg i}(w_{t}),\widehat{F}^{\neg i}(w_{t}^{\neg i})\right\}.\nn
\eea
In turn applying this back in \eqref{eq:wt - wti} we have shown that
\bea
 \Big\|w_{t+1}-w_{t+1}^{\neg i}\Big\|  \le \left(1+\eta \frac{2LR^2}{\sqrt{m}} \max\left\{\widehat{F}^{\neg i}(w_{t}),\widehat{F}^{\neg i}(w_{t}^{\neg i})\right\}\right) \Big\|w_{t}-w_t^{\neg i}\Big\|  + \frac{\eta\ell R}{n} \hf_i(w_t)\label{eq:wt - wti good}
\eea

To continue, denote for convenience 
\[\beta^{i}_t:= \eta \frac{2LR^2}{\sqrt{m}} \max\left\{\hf^\negi(w_{t}),\hf^\negi(w_{t}^\negi)\right\}\quad\text{and}\quad\rho:=\eta\ell R,\] so that:
\[
  \Big\|w_{t+1}-w_{t+1}^{\neg i}\Big\| \le  \left(1+\beta_t^i \right) \Big\|w_{t}-w_t^{\neg i}\Big\|  + \frac{\rho}{n}\hf_i(w_t),\qquad\forall i\in[n],t\in[T]\,.
\]
By unrolling the iterations over $t\in[T]$ and noting $w_0=w_0^{\neg i}$, we obtain the following for the  leave-one-out parameter distance at iteration $T$:
\bea
  \Big\|w_{T}-w_{T}^{\neg i}\Big\| &\le \frac{\rho}{n}\sum_{t=0}^{T-1}  \left(\prod_{\tau=t+1}^{T-1}(1+\beta_\tau^i)\right) \widehat F_i(w_t)\nn\\
&\le \frac{\rho}{n} \sum_{t=0}^{T-1} \exp\left(\sum_{\tau=t+1}^{T-1}\beta_\tau^i\right) \hf_i(w_t)\nn\\
&\le  \frac{\rho}{n} \sum_{t=0}^{T-1} \exp\left(\sum_{\tau=1}^{T-1}\beta_\tau^i \right) \widehat F_i(w_t) = \exp\left(\sum_{\tau=1}^{T-1}\beta_\tau^i \right) \frac{\rho}{n} \sum_{t=0}^{T-1}  \widehat F_i(w_t)
\nn\\
&\le   \frac{\rho }{n} \,\exp\left({\max_{j\in[n]}\sum_{\tau=1}^{T-1}\beta_\tau^j}\right)\,\sum_{t=0}^{T-1}  \widehat F_i(w_t),\qquad \forall i\in[n]\,.
\label{eq:LOOBT}
\eea
It remains to bound $\beta:=\max_{i\in[n]}\sum_{\tau=1}^{T-1}\beta_\tau^i$. We do this as follows:
\begin{align}
\beta&=\frac{2\eta LR^2}{\sqrt{m}}\max_{i\in[n]}\left\{\max\left\{\sum_{t=1}^T\Fh^\negi(w_t)\,,\,\sum_{t=1}^T\Fh^\negi(w_t^\negi)\right\}\right\}\nn
\\
&\leq \frac{2\eta LR^2}{\sqrt{m}}\max_{i\in[n]}\left\{\max\left\{\sum_{t=1}^T\Fh(w_t)\,,\,\sum_{t=1}^T\Fh^\negi(w_t^\negi)\right\}\right\}\nn
\\
&= \frac{2\eta LR^2}{\sqrt{m}}\max\left\{\sum_{t=1}^T\Fh(w_t)\,,\,\max_{i\in[n]}\sum_{t=1}^T\Fh^\negi(w_t^\negi)\right\}\nn
\\
&= \frac{2\eta LR^2}{\sqrt{m}} T\,\max\left\{\regret\,,\,\regret_\loo\right\}\leq 2/3\,,\nn
\end{align}
where: (i) in the first inequality we used nonnegativity of $f(\cdot)$ to conclude for any $i\in[n]$ and any $w$ that $\Fh^\negi(w)\leq \Fh(w)$; (ii) in the last line, we recalled the definition of the regret terms and we used the theorem's condition \eqref{eq:m large model stability 2} on large enough $m.$

Using this in \eqref{eq:LOOBT} and averaging over $i\in[n]$ yields
\begin{align*}
\frac{1}{n}\sum_{i\in[n]}\Big\|w_{T}-w_{T}^{\neg i}\Big\| &\leq \frac{\rho e^\beta}{n}\sum_{t=0}^{T-1}\frac{1}{n}\sum_{i=1}^{n}\Fh_i(w_t)
\\
&\leq \frac{\eta\ell R e^{2/3}}{n}\sum_{t=0}^{T-1}\Fh(w_t)\,.
\end{align*}
The advertised bound follows by using $e^{2/3}\leq 2$ and writing \[\frac{1}{T}\sum_{t=0}^{T-1}\Fh(w_t)\leq\frac{1}{T}\sum_{t=0}^{T}\Fh(w_t)= \frac{\Fh(w_0)}{T}+\regret.\]
\end{proof}


To bound the generalization gap in terms of model stability we rely on the following result.
\begin{lemma}[\cite{lei2020fine}]\label{lem:sk22}
   Suppose the sample loss $f(\cdot,z)$ is $G_\hf$-Lipschitz  for almost surely all data points $z\sim\mathcal{D}$. Then, the following relation holds between expected generalization loss and model stability at any iterate $T$,
   \begin{align}
       \E\Big[{F}(w_T)\Big] - \E\Big[\widehat F(w_T)\Big] \leq  2G_\hf \;\E\Big[\frac{1}{n}\sum_{i=1}^n\|w_T-w_T^{\neg i}\|\Big].
   \end{align}
\end{lemma}
With the two results above, we are ready to prove Theorem \ref{thm:generalization gap general}.
\begin{theorem}[Restatement of Theorem \ref{thm:generalization gap general}]
Suppose Assumptions \ref{ass:features}- \ref{ass:self boundedness} hold. Fix any time horizon $T\geq 1$ and any step size $\eta\leq 1/L_\Fh$ where  $L_\hf$ is the objective's smoothness parameter. Let any $w\in\R^\dpr$ such that $\|w-w_0\|^2\geq \max\{\eta T\, \Fh(w),\eta \Fh(w_0)\}.$ Suppose  hidden-layer width $m$ satisfies 
$
m\geq 64^2 L^2 R^4 \|w-w_0\|^4.
$
Then, the generalization gap of GD at iteration $T$ is bounded as
\[
    \E\Big[F(w_T)-\Fh(w_T)\Big] \leq \frac{8\ell^2 R^2}{n}\, \E\left[ \eta T \,\Fh(w) + {2\|w-w_0\|^2}\right],\,\,
\]
where all expectations are over the training set.
\end{theorem}
\begin{proof}
    The proof essentially follows by combining Theorem \ref{thm:model_stability} with Theorem \ref{thm:train}. Note that the assumptions of Theorem \ref{thm:train} are met. Thus, the regret and parameter-norm are bounded as follows:
    \begin{align}\label{eq:in proof reg and norm}
    \regret\leq 2\Fh(w) + \frac{5\|w-w_0\|^2}{2\eta T} \qquad\text{and}\qquad \max_{t\in[T]}\;\|w_t-w_0\|\;\leq 4\|w-w_0\|\,.
    \end{align}
    We can also use Theorem \ref{thm:train} to the leave-one-out objective $\Fh^\negi$ and the corresponding loo GD updates $w_t^\negi$. This bounds the loo regret and the norm of the loo parameter, as follows:
    \[
    \regret_\loo\leq 2\Fh(w) + \frac{5\|w-w_0\|^2}{2\eta T} \qquad\text{and}\qquad \max_{i\in[n]}\max_{t\in[T]}\;\|w_t^\negi-w_0\|\;\leq 4\|w-w_0\|\,.
    \]
    
    We use these two displays to show that $m$ is by assumption large enough so that Eqs. \eqref{eq:m large model stability 1} and \eqref{eq:m large model stability 2} hold. Indeed, we have
\begin{align*}
    \sqrt{m}\geq 64LR^2\|w-w_0\|^2 = 4LR^2 \left(4\|w-w_0\|\right)^2 \geq 4LR^2\max\left\{\|w_t-w_0\|^2,\|w_t^\negi-w_0\|^2\right\} 
\end{align*}
and
\begin{align*}
    \sqrt{m}\geq 64LR^2\|w-w_0\|^2 &> 6LR^2 \cdot 5\|w-w_0\|^2 \\&> 6LR^2 \cdot \big(2\eta T\Fh(w)+5\|w-w_0\|^2/2\big) \\ &\geq
    6LR^2\eta T\max\left\{\regret,\regret_\loo\right\}\,.
    \label{eq:m large model stability 2}
\end{align*}
In the second display we also used the theorem's assumption that $\|w-w_0\|^2\geq \eta T \Fh(w)$.

Thus, we can apply Theorem \ref{thm:model_stability} to find that
\begin{align*}
       \frac{1}{n}\sum_{i=1}^n\Big\|w_T-w_T^{\neg i}\Big\| &\leq \frac{2\ell R}{n}  \left(\eta\Fh(w_0) + \eta T\cdot\regret \right)
       \\
       &\leq \frac{2\ell R}{n}  \left(\eta\Fh(w_0) + 2\eta T\Fh(w) + 5\|w-w_0\|^2/2 \right)
              \\
       &\leq \frac{2\ell R}{n}  \left(2\eta T\Fh(w) + 7\|w-w_0\|^2/2 \right)
   \end{align*}
   where in the penultimate line we used \eqref{eq:in proof reg and norm} and in the last line we used the theorem's assumption that $\|w-w_0\|^2\geq \eta \Fh(w_0)$. 

   To conclude the proof, simply take expectations over the train set on the above display and apply Lemma \ref{lem:sk22} recalling $G_\Fh=\ell R.$
\end{proof}

\subsection{Proof of Theorem \ref{thm:gengap}}
\label{sec:gengapapp}
Here we prove the generalization gap for interpolating neural networks as per Theorem \ref{thm:gengap}. 

\begin{theorem}[Restatement of Theorem \ref{thm:gengap}]\label{thm:gengapapp}
   Let Assumptions \ref{ass:features}-\ref{ass:real} hold. Fix $T\ge 1$ and let $m\ge 64^2 L^2R^4\, g(\frac{1}{T})^4$. Then, for any $\eta\le\min\{\frac{1}{L_\hf},g(1)^2,\frac{g(1)^2}{\hf(w_0)}\}$ the expected generalization gap at iteration $T$ satisfies 
    \bea
  \E \Big[{F}(w_T)-\widehat F(w_T)\Big]\le \frac{24\ell^2 R^2 \,g(\frac{1}{T})^2 }{n} \,.
    \eea
\end{theorem}
\begin{proof}
According to Assumption \ref{ass:real}, for any sufficiently small $\eps>0$, there exists $w^{(\eps)}$ such that $\widehat F(w^{(\eps)})\le\eps$ and $\|w^{(\eps)}-w_0\| = g(\eps)$. Recall from Theorem \ref{thm:generalization gap general} that, 
\bea\label{eq:thm26}
    \E\left[F(w_T)-\Fh(w_T)\right] \leq \frac{8\ell^2 R^2}{n}\, \left( \eta T \Fh(w) + {2\|w-w_0\|^2}\right)\,.
\eea
In particular let $\eps=1/T$ and replace $w$ with $w^{(\eps)}$. This is possible since after $T\ge 1$ steps and with the decreasing nature of $g$ and the condition on step-size it holds that $\|w^{(1/T)}-w_0\|^2 = g(1/T)^2\ge g(1)^2 \ge \max\{\eta T \hf(w^{(1/T)}),\eta \hf(w_0)\}$. Thus continuing from \eqref{eq:thm26} we have,
\[
    \E\left[F(w_T)-\Fh(w_T)\right] \leq \frac{8\ell^2 R^2}{n}\, \left( \eta + 2g(\frac{1}{T})^2 \right)\,.
\]
Recalling $\eta\le g(1)^2\le g(\frac{1}{T})^2$ leads to the claim of the theorem.
\end{proof}

\section{Proofs for Section \ref{sec:ntk}}\label{sec:ntkapp}

We first prove proposition \ref{propo:ntksepreal}, which we repeat here for convenience.

\begin{proposition}[Restatement of Proposition \ref{propo:ntksepreal}] \label{propo:ntkseprealapp}
 Let Assumptions \ref{ass:features}-\ref{ass:act},\ref{ass:ntksep}-\ref{ass:init} hold. Assume $f(\cdot)$ to be the logistic loss. Fix $\eps>0$ and let $m\ge \frac{L^2 R^4}{4\gamma^4 C^2} (2C+\log(1/\eps))^4$. Then the realizability Assumption \ref{ass:real} holds with $g(\eps)=\frac{1}{\gamma}(2C+\log(1/\eps))$. In other words, there exists $w^{(\eps)}$ such that
 \bea
 \hf(w^{(\eps)})\le \eps, \;\;\;\text{and}\;\;\; \left\|w^{(\eps)}-w_0\right\| = \frac{1}{\gamma}\left(2C+\log(1/\eps)\right).
 \eea
\end{proposition}
\begin{proof}
    By Taylor there exists $w'\in[w,w_0]$ such that,
    \bea\label{eq:taylorsep}
    y_i\Phi(w,x_i) = y_i \Phi(w_0,x_i) + y_i \Big\langle \nabla_1\Phi(w_0,x_i),w-w_0\Big\rangle+ \frac{1}{2} y_i \Big\langle w-w_0,\nabla^2_1 \Phi(w',x_i)(w-w_0)\Big\rangle
    \eea
    Pick $w=w^{(\eps)}:= w_0+\frac{w^\star}{\gamma}(2C+\log(1/\eps))$ for $w^\star$ defined in Assumption \ref{ass:ntksep}. Since $\|w^\star\|=1$, we automatically derive the desired for $\|w^{(\eps)}-w_0\|$. Next, we show that $\hf_i(w^{(\eps)})\le \eps.$ Based on Lemma \ref{lem:model gradients}, $\|\nabla^2_1\Phi(w',x_i)\|\le \frac{LR^2}{\sqrt{m}}.$ Continuing from Eq. \eqref{eq:taylorsep}, we deduce the following,
\begin{align*}
y_i\Phi(w,x_i) &\ge -\left|y_i \Phi(w_0,x_i) \right| + y_i \left\langle \nabla_1\Phi(w_0,x_i),w^{(\eps)}-w_0\right\rangle - \frac{1}{2} \Big\|\nabla^2_1 \Phi(w',x_i)\Big\| \left\|w^{(\eps)}-w_0\right\|^2\\
&\ge -C + 2C + \log(1/\eps) - \frac{LR^2}{2\gamma^2\sqrt{m}}(2C+\log(1/\eps))^2\\
&\ge \log(1/\eps).
\end{align*}
The last step is due to the condition on $m$. The inequality above implies that $\hf_i(w):=f(y_i\Phi(w,x_i))\le \log(1+\eps)\le\eps$, and thus $\hf(w)\le \eps$ as desired. This completes the proof.
\end{proof}

With this, we many now prove Corollary \ref{cor:ntkres}.
\begin{corollary}[Restatement of Corollary \ref{cor:ntkres}]
       Let Assumptions \ref{ass:features}-\ref{ass:act},\ref{ass:ntksep}-\ref{ass:init} hold and assume logistic loss. 
    Suppose $m\ge \frac{64^2L^2R^4}{\gamma^4}(2C+\log(T))^4$ for a fixed training horizon $T$. Then, for any $\eta\le\min\{3,\frac{1}{L_\hf}\}$ the training loss and generalization gap are bounded as follows:
    \begin{align}
    &\hf(w_T) \le \frac{5(2C+\log(T))^2}{\gamma^2 \eta T},\nn\\
    &\E\Big[F(w_T)-\hf(w_T)\Big] \le \frac{24\ell^2 R^2}{\gamma^2 n}(2C+\log(T))^2.\nn
    \end{align}
\end{corollary}
\begin{proof}
    The given assumption on $m$ satisfies the conditions of Proposition \ref{propo:ntksepreal} for $\eps=\frac{1}{T}$, $g(1/T)=\frac{1}{\gamma}(2C+\log(T))$. We can apply the results of our optimization and generalization results from Theorems \ref{thm:train_IP} and \ref{thm:gengap} for a fixed $T$ which satisfies $T\ge 1$. Note that we can assume without loss of generality that $\gamma \le 1$ which implies that $g(1)^2=4C^2/\gamma^2\ge 4$. Moreover, for logistic loss it holds $g(1)^2/\hf(w_0) \ge  \frac{4C^2}{\gamma^2 \,\log (1+e^{C})} \ge 3$ for all $C\ge 1$. Therefore the condition on step-size simplifies to $\eta\le\min\{3,1/L_\hf\}$. This completes the proof.  
\end{proof}

\subsection{Proof of Proposition \ref{propo:marginntksep}}
The proof of Proposition \ref{propo:marginntksep} has the following steps: First, we consider an infinite-width NTK separability assumption (Assumption \ref{ass:infntk}) and show in Lemma \ref{lem:ntktoinft} that it is equivalent with high-probability to the NTK-separability in Assumption \ref{ass:ntksep} given logarithmic number of neurons. We then prove that the noisy-XOR dataset satisfies Assumption \ref{ass:infntk} for convex and locally strongly-convex activations. The result of Proposition \ref{propo:marginntksep} then follows by combining the two lemmas. 
\begin{ass}[Infinite-width NTK-separability]\label{ass:infntk} There exists $\overline{w}(\cdot): \mathbb{R}^d \rightarrow \mathbb{R}^d$ and $\gamma > 0$ such that $\|\overline{w}(z)\|_2 \le 1$ for all $z \in \mathbb{R}^d$, and for all $(x,y)\sim \mathcal{D}$,
$$
y \int_{\mathbb{R}^d} \sigma^{\prime}\left(\left\langle z, x\right\rangle\right) \cdot\left\langle\overline{w}(z), x\right\rangle \mathrm{d} \mu_{\mathrm{N}}(z) \ge \gamma,
$$
where $\mu_N(\cdot)$ denotes the standard Gaussian measure.
\end{ass}

\begin{lemma}\label{lem:ntktoinft}
    Let $\{(x_i,y_i)\}$ be any dataset of size $\tilde{n}$ under Assumption \ref{ass:features}, satisfying the separability condition of Assumption \ref{ass:infntk} with some margin $\tilde\gamma>0$. Consider initialization $w_0\in\R^{d'}$ where $w_{0}\sim N(0,I_{d'})$. Then, with probability at least $1-\delta$ the dataset is separable under Assumption \ref{ass:ntksep} with margin at least $\gamma=\tilde\gamma - \frac{\ell R}{\sqrt{2m}} \log^{1/2}(\tilde{n}/\delta)$, i.e., there exists unit norm $w^\star$ such that for all $i\in[\tilde{n}]: y_i \langle \nabla_1\Phi(w_0,x_i),w^\star\rangle \ge \gamma.$
\end{lemma}
\begin{proof}
    By the model's gradient we have for any $w^\star\in\R^{d'}$,
    \bea
    \phi_{i}:= y_i \Big\langle \nabla_1\Phi(w_0,x_i),w^\star\Big\rangle = y_i \sum_{j=1}^m \frac{a_j}{\sqrt{m}} \sigma'(\langle w_{0,j},x_i\rangle) \langle x_i,w^\star_j\rangle.
    \eea
    Let $w_j^\star = \frac{a_j}{\sqrt{m}}\overline{w} (w_{0,j})$. Then $\|w^\star\|\le 1$ and by Hoeffding's inequality
    it holds for all $t\ge0$,
    \bea
    \Pr\Big(\phi_i \ge \tilde\gamma -t\Big) \ge 1-\exp\left(\frac{-2t^2 m}{\ell^2 R^2}\right).
    \eea
    This leads to the desired result with an extra union bound over $i\in[\tilde{n}]$.
\end{proof}

\begin{lemma}
    Consider the noisy XOR data distribution $\{(\bar x_i,y_i)\}$ and two-layer neural network with a convex activation which is $\mu$-strongly convex in $[-2,2]$ i.e., $\min_{t\in[-2,2]} \sigma''(t) \ge \mu$ for some $\mu>0$. Then the separability assumption \ref{ass:infntk} is satisfied with margin $\gamma = \frac{\mu}{40 d}.$
\end{lemma}
\begin{proof}
The proof is essentially similar to \cite[Prop. 5.3]{Ji2020Polylogarithmic} and thus we follow their notation and omit the details for brevity. While  their proof relies rather crucially on the ReLU activation, it can be appropriately modified to obtain a similar margin bound under our different assumptions on the activation function. To see this, note that due to convexity of activation function, the integrand in the line above Eq. (D.4) is non-negative. Therefore, we can lower-bound the integral (which evaluates the margin) by restricting $A_1$ to $|p_1|<1$. With this restriction we can use the local strong convexity of activation function to lower-bound the margin, i.e., to uniformly lower-bound $y_i \int_{\mathbb{R}^d} \sigma^{\prime}\left(\left\langle z, x_i\right\rangle\right) \cdot\left\langle\bar{w}(z), x_i\right\rangle \mathrm{d} \mu_{\mathrm{N}}(z)$ for all $i\in[n]$. Specifically, note that with strong convexity in $[-2,2]$, Eq. (D.4) in \cite{Ji2020Polylogarithmic} changes to $\geq\frac{2p_1}{d-1} U(p_1) \min_{t\in[-2,2]} \sigma''(t) \ge \frac{2p_1\mu}{d-1} U(p_1)$ where $U(t):=\int_{-t}^t \varphi(\tau) \mathrm{d} \tau$ is the probability that a standard Gaussian random variable falls in $[-t,t]$. This leads to the final value for margin being $\frac{2\mu}{d-1} \int p_1\,U(p_1)\, \one \left[p \in A_1\right] \mathrm{d} \mu_{N}(p) \ge \frac{8\mu}{(2\pi e)^{3/2}(d-1)}\int_0^1 p_1^3 \mathrm{d}{p_1} \ge \frac{\mu}{40d},$ as desired.
\end{proof}

\begin{proposition}[Restatement of Proposition \ref{propo:marginntksep}]
   Consider the noisy XOR data distribution $\{(\bar x_i,y_i)\}$. Assume the activation function is convex, $\ell$-Lipschitz and $\mu$-strongly convex in the interval $[-2,2]$ for some $\mu>0$, i.e., $\min_{t\in[-2,2]} \sigma''(t) \ge \mu$. Moreover, assume Gaussian initialization $w_0\in\R^{d'}$ with entries iid $N(0,1)$. If $m \ge \frac{80^2 d^3 \ell^2}{2\mu^2} \log (2/\delta)$, then with probability at least $1-\delta$ over the initialization, the NTK-separability Assumption \ref{ass:ntksep} is satisfied with margin $\gamma = \frac{\mu}{80 d}$. 
\end{proposition}
\begin{proof}
    The claim follows by combining the last two lemmas. In particular, we derive the infinite width NTK-separability for the entire data distribution (of size $2^d$) with margin $\tilde\gamma=\frac{\mu}{40d}$ and by the assumption on width and noting $\tilde{n}=2^d$, we have $\gamma$-separability by NTK for the entire distribution with probability $1-\delta$ where $\gamma=\tilde\gamma - \frac{\ell R}{\sqrt{2m}} \log^{1/2}(\tilde{n}/\delta)=\frac{\mu}{40d}- \frac{\ell R \sqrt{d}}{\sqrt{2m}} \log^{1/2}(1/\delta)\ge \frac{\mu}{80d}$. This completes the proof.
\end{proof}
Finally, we show how to control the parameter $C$ that bounds the model output at Gaussian initialization.
\begin{lemma}[Initialization bound]\label{lem:init} 
Let Assumption \ref{ass:features} hold and assume the activation function to be $\ell$-Lipschitz. Consider initialization $w_0\in\R^{d'}$ where $w_{0}\sim N(0,I_{d'})$. Given any $\delta \in(0,1)$, then with probability at least $1-\delta$, it holds for all $i\in[\tilde{n}]$ that
\bea
\left|\Phi\left( w_0, x_i\right)\right| \leq \ell R\sqrt{2 \log (2 \tilde{n} / \delta)}.
\eea
\end{lemma}
\begin{proof}
    Recall that if a function $\phi:\R^{d'}\rightarrow\R$ is $G$-Lipschitz then for Gaussian vector $Z=(Z_1,Z_2,\cdots,Z_{d'})$ where each component is i.i.d. standard Gaussian $Z_i\sim N(0,1)$, it holds for all $t\ge0$ that $\Pr[|\phi(Z)-\mathbb{E}[\phi(Z)]| \geq t] \leq 2 \exp(-\frac{t^2}{2 G^2})$. Note that according to Lemma \ref{lem:model gradients}, $\Phi(\cdot,x_i)$ is $(\ell R)$-Lipschitz for any data point $x_i$. Therefore, with the given initialization for $w_0$, we have
    \bea
    \Pr \left[\Big|\Phi(w_0,x_i)-\E[\Phi(w_0,x_i)]\Big|\ge t\right]\le 2 \exp\left(-\frac{t^2}{2\ell^2 R^2}\right).\nn
    \eea
It also holds that $\E[\Phi(w_0,x_i)]=0$. This is true since for half of second layer weights $a_j=1$ and for the rest $a_j=-1.$ Thus, we have $\Pr \left[|\Phi(w_0,x_i)|\ge t\right]\le 2 \exp(-\frac{t^2}{2\ell^2 R^2})$. A union bound yields that uniformly over $i\in[\tilde{n}]$, we have $\Pr \left[|\Phi(w_0,x_i)|\ge t\right]\le 2\tilde{n}\cdot\exp(-\frac{t^2}{2\ell^2 R^2})$ which concludes the claim of lemma. 
\end{proof}

\section{Gradients and Hessian calculations}

\subsection{Definitions}
Assume \iid data $(x,y)\sim\Dc$, $x\in\R^{d}, y\in\{\pm1\}$. Denote for convenience $z:=yx$. Suppose two-layer neural network model
\begin{align}\label{eq:model}
\Phixi=\frac{1}{{\sqrt{m}}}\sum_{j\in[m]}a_j\sigma(\inp{w_j}{x})
\end{align}
$a_j\in\{\pm1\}, j\in[m]$ and first-layer weights trained by GD on
\begin{align}\label{eq:train-losss}
\widehat F(w)=\frac{1}{n}\sum_{i\in[n]} f(y_i\Phixi)=:\frac{1}{n}\sum_{i\in[n]} f(w,z_i)\,.
\end{align}
for loss function $f:\R\rightarrow\R$.

For convenience define
\begin{subequations}
\begin{align}\label{eq:Fprime}
\Fh^\prime(w)&=\frac{1}{n}\sum_{i\in[n]} \abs{f^\prime(y_i\Phixi)}\\
\Fh^{\prime\prime}(w)&=\frac{1}{n}\sum_{i\in[n]} \abs{f^{\prime\prime}\,(y_i\Phixi)}\label{eq:Fprimeprime}
\end{align}
\end{subequations}

\subsection{Model's Gradient/Hessian}
\begin{lemma}\label{lem:model gradients}
The following are true for the model \eqref{eq:model} under Assumption \ref{ass:act}.
    \begin{enumerate}
\item $
\|\DPhix\| \leq \ell R$.    

\item $\|\DDPhix\| \leq \frac{L R^2}{{\sqrt{m}}}$.     \end{enumerate}
\end{lemma}
\begin{proof}
Direct calculation yields that,
\begin{align*}
    \DPhix &= \frac{1}{{\sqrt{m}}}\begin{bmatrix}
    a_1\sigma^\prime(\inp{w_1}{x}) x \\ \cdot\\ \cdot\\       a_m\sigma^\prime(\inp{w_m}{x}) x
    \end{bmatrix}
\end{align*} 
Noting that $\sigma'(\cdot)\le\ell$,
\bea
\|\nabla_1 \Phi (w,x)\|^2 &= \frac{1}{m}\sum_{j=1}^m \sum_{i=1}^d (x(i) \sigma'(\langle w_j,x\rangle))^2 \\
&\le \ell^2\|x\|^2 \nn\\
&\le \ell^2 R^2.\nn
\eea
For the Hessian,
\bea
\frac{\partial^2 \Phi(w,x)}{\partial w_{ij}\partial w_{k\ell}} = \frac{1}{{\sqrt{m}}} x(j)x(\ell) a_i\sigma''(\langle w_i,x\rangle)\one_{\{i=k\}}.
\eea
Thus,
\begin{align*}
    \DDPhix &= \frac{1}{{\sqrt{m}}}\operatorname{diag}\left(a_1\sigma^{\prime\prime}(\inp{w_1}{x}) xx^T,\ldots,a_m\sigma^{\prime\prime}(\inp{w_m}{x}) xx^T\right)
    \end{align*}
    
for any unit norm vector $u\in\R^{md}$, define $\bar u_i:= [u_{(i-1)m+1}:u_{im}]\in\R^{d}$. Moreover, define the matrix $\nabla^2_{w_i}\Phi(w,x)\in\R^{d\times d}$ such that $[\nabla^2_{w_i}\Phi(w,x)]_{j\ell}=\frac{\partial^2 \Phi(w,x)}{\partial w_{ij}\partial w_{i\ell}}$
\begin{align*}
\Big\|u^\top\nabla^2_1 \Phi(w,x)\Big\|^2 &= \sum_{i=1}^m \Big\|u_i^\top \nabla_{w_i}^2 \Phi(w,x)\Big\|^2\\
&\le \sum_{i=1}^m \Big\|\nabla^2_{w_i} \Phi(w,x)\Big\|^2\|\bar u_i\|^2\\
&\le  \sum_{i=1}^m \frac{L^2}{m} \|x\|^4 \|\bar u_i\|^2\\
&\le \frac{L^2 R^4}{m}. 
\end{align*}
This completes the proof.

\end{proof}
\subsection{Objective's Gradient/Hessian}
\begin{lemma}\label{lem:gradients_and_hessians_bounds_general}
Let Assumption \ref{ass:act} hold. 
Then, the following are true for the loss gradient and Hessian:
\begin{enumerate}
  
\item $\|\nabla \widehat F(w)\|\leq \ell R\,\Fh^\prime(w).$

\item $\|\nabla^2 \widehat F(w)\| \leq \ell^2 R^2 \Fh^{\prime\prime}(w) + \frac{L R^2}{{\sqrt{m}}} \Fh^\prime(w).$

\item $\lambda_\min\left(\nabla^2 \widehat F(w)\right) \geq - \frac{LR^2}{{\sqrt{m}}}\Fh^\prime(w)$.
\end{enumerate}
\end{lemma}

\begin{proof} 
The loss gradient is derived as follows,
    \begin{align*}
    \nabla \widehat F(w) &= \frac{1}{n} \sum_{i=1}^n f^\prime(y_i \Phixi)y_i \DPhixi
    \end{align*}
    Recalling that $y_i\in\{\pm 1\}$, we can write
\bea
\Big\|\nabla \widehat F(w)\Big\| &= \frac{1}{n} \Big\| \sum_{i=1}^n f'(y_i \Phi(w,x_i))y_i \nabla_1 \Phi (w,x_i)\Big\|\nn\\
&\le \frac{1}{n} \sum_{i=1}^n |f'(y_i \Phi(w,x_i))| \Big\|\nabla_1 \Phi (w,x_i)\Big\|.\nn\\
&\le \ell R\, F'(w). 
\eea
For the Hessian of loss, note that
    \begin{align}\label{eq:lossh}
    \nabla^2 \widehat F(w) &= \frac{1}{n} \sum_{i=1}^n f^{\prime\prime}(y_i \Phixi) \DPhixi\DPhixi^\top + 
     f^{\prime}(y_i \Phixi)y_i \DDPhixi.
    \end{align}
It follows that
\bea
\Big\|\nabla^2 \widehat F(w)\Big\| &= \left\|\frac{1}{n} \sum_{i=1}^n f^{\prime}(y_i \Phi(w,x_i))y_i \nabla^2_1 \Phi (w,x_i) + f^{\prime\prime}(y_i \Phi(w,x_i)) \nabla_1 \Phi(w,x_i) \nabla_1 \Phi(w,x_i)^\top \right\|\nn\\
&\le \frac{1}{n}\sum_{i=1}^n |f'(y_i \Phi(w,x_i))| \Big\|\nabla^2_1 \Phi (w,x_i) \Big\| + |f''(y_i \Phi(w,x_i))|\Big\|\nabla_1 \Phi (w,x_i) \nabla_1\Phi(w,x_i)^\top \Big\|\nn\\
&\le  \frac{1}{n}\sum_{i=1}^n |f'(y_i \Phi(w,x_i))| \Big\|\nabla^2_1 \Phi (w,x_i) \Big\| + |f''(y_i \Phi(w,x_i))|\Big\|\nabla_1 \Phi (w,x_i)\Big\|^2\nn\\
&\le \frac{LR^2}{{\sqrt{m}}} F'(w)+ \ell^2 R^2 F''(w).
\label{eq:heslast} 
\eea
To lower-bound the minimum eigenvalue of Hessian, note that $f$ is convex and thus $f''(\cdot)\ge 0$. Therefore the first term in \eqref{eq:lossh} is positive semi-definite and the second term can be lower-bounded as follows,
\bea
\lambda_{\min} (\nabla ^2 \widehat F(w)) &\ge -\left \|\frac{1}{n} \sum_{i=1}^n y_i f'(y_i\Phi(w,x_i)) \nabla_1^2 \Phi(w,x_i)\right\|\nn\\
&\ge -\frac{1}{n} \sum_{i=1}^n |y_i f'(y_i\Phi(w,x_i))| \Big\|\nabla_1^2 \Phi(w,x_i)\Big\|\nn\\
&\ge -\frac{LR^2}{{\sqrt{m}}} F'(w).\nn
\eea
    
\end{proof}


\begin{corollary}[Self-boundedness of Objective]\label{cor:gradients_and_hessians_bounds}
Let Assumption \ref{ass:act} hold. 
\\
If the loss satisfies Assumptions \ref{ass:self boundedness} (with $\beta_f=1$) and \ref{ass:2nd order self-boundedness}, then
\begin{enumerate}
  
\item $\|\nabla \widehat F(w)\| \leq \ell R\, \hf(w)$.

\item $\|\nabla^2 \widehat F(w)\|\leq \left(\ell^2 R^2 + \frac{L R^2}{{\sqrt{m}}}\right) \widehat F(w)$.

\item $\lambda_\min\left(\nabla^2 \widehat F(w)\right) \geq - \frac{LR^2}{{\sqrt{m}}}\widehat F(w)$.
\end{enumerate}

\noindent If in addition the loss satisfies Assumptions \ref{ass:loss lipschitz} and \ref{ass:loss smooth} with $L_f=G_f=1$, then
\begin{enumerate}
\setcounter{enumi}{5}
    \item $\|\nabla \widehat F(w)\| \leq \ell R .$

\item $\|\nabla^2 \widehat F(w)\| \leq \ell^2 R^2 + \frac{L R^2}{{\sqrt{m}}}.$
\end{enumerate}
\end{corollary}

\begin{proof}
     For self-bounded losses we have $\hf'(w)\le \hf(w)$ and $\hf''(w)\le \hf(w)$. If the loss is $1$-Lipschitz and $1$-smooth we have $\hf'(w)\le 1$ and $\hf''(w)\le 1$. Thus, the claims immediately follow from Lemma \ref{lem:gradients_and_hessians_bounds_general}.
\end{proof}

\section{Detailed technical comparison to most-closely related works}\label{sec:more related}
 In terms of techniques, the most closely related works to our paper  are the recent works \cite{richards2021learning,richards2021stability,leistability}, which also utilize the stability-analysis framework to derive test-loss bounds of GD for shallow neural networks. 

\cite{richards2021learning} investigates the generalization gap of weakly-convex losses for which $\lambda_{\min}(\nabla^2 \widehat{F}(w))\geq - \epsilon$ for a constant $\epsilon>0.$ Note by Lemma \ref{lem:gradients_and_hessians_bounds_general} that our empirical loss is weakly convex with $\epsilon=LR^2/\sqrt{m}$ since the logistic loss is $1$-Lipschitz. Within the stability analysis framework, \cite{richards2021learning} leverage the weak-convexity property to establish an approximate expansiveness property of GD iterates that in our setting translates to
\begin{align}
    \left\|\left(w-\eta\nabla\widehat{F}(w)\right)-\left(\wpr-\eta\nabla\widehat{F}(\wpr)\right)\right\|\lesssim\left(1+\frac{\eta LR^2}{\sqrt{m}}\right)\left\|w-\wpr\right\|\,.
    \label{eq:RichardsRabbat}
\end{align}
When using this inequality to bound the model stability term at iteration $t$, and in order to obtain non-vacuous bounds, the extra term in \eqref{eq:RichardsRabbat} must be chosen such that ${\eta LR^2}/{\sqrt{m}}\lesssim 1/t$. This leads to polynomial-width parameterization requirement $m\gtrsim t^2$. In this work, we reduce the requirement to logarithmic  $m\gtrsim \log(t)$, by significantly tightening \eqref{eq:RichardsRabbat}. This is achieved by introducing two crucial ideas. The first is to exploit the self-boundedness property of loss function, which yields a stronger \emph{self-bounded} weak convexity $\lambda_{\min}(\nabla^2 \widehat{F}(w))\geq - LR^2\widehat{F}(w)/\sqrt{m}.$ With this, we show in Corollary \ref{cor:9.1} that
\begin{align}
\left\|\left(w-\eta\nabla\widehat{F}(w)\right)-\left(\wpr-\eta\nabla\widehat{F}(\wpr)\right)\right\|\leq\left(1+\frac{\eta LR^2}{\sqrt{m}}\max_{\alpha\in[0,1]}\widehat{F}(\wa)\right)\left\|w-\wpr\right\|\,\label{eq:Hossein}
\end{align}
for some $\wa=\alpha w+(1-\alpha)\wpr$. Our second idea comes into bounding the term $\max_{\alpha\in[0,1]}F(\wa)$ which in our bound replaces the Lipschitz constant $G_f$  of \eqref{eq:RichardsRabbat}. To control $\max_{\alpha\in[0,1]}F(\wa)$, we identify and use the Generalized Local Quasi-convexity of Proposition \ref{propo:GLQC}. This replaces $\max_{\alpha\in[0,1]}F(\wa)$ in \eqref{eq:Hossein} with $\tau\cdot \max\{\widehat F(w),\widehat F(\wpr)\}$ for $\tau\approx 1+LR^2\|w-w'\|^2/\sqrt{m}$ and note that we can guarantee $\tau=O(1)$  provided $\sqrt{m}\gtrsim \max\{\|w-w_0\|^2,\|\wpr-w_0\|^2\}.$ Now, in order to bound the model stability term, we apply non-expansiveness for GD iterate $w=w_t$ and its leave-one-out counterpart $w=w_t^\negi$: Provided $m\gtrsim \max\{\|w_t-w_0\|^4,\|w_t^\negi-w_0\|^4\}\approx \log^4(t),$
\begin{align}
\left\|\left(w_t-\eta\nabla\widehat{F}(w_t)\right)-\left(w_t^\negi-\eta\nabla\widehat{F}(w_t^\negi)\right)\right\|&\lesssim\left(1+\frac{\eta LR^2}{\sqrt{m}} \max\{\widehat F(w),\widehat F(\wpr)\} \right)\left\|w-\wpr\right\|\,\nn
\\
&\lesssim\left(1+\frac{\eta LR^2}{t\,\sqrt{m}}  \right)\left\|w-\wpr\right\|\
\label{eq:Hossein2}
\end{align}
Compared to \eqref{eq:RichardsRabbat} note in \eqref{eq:Hossein2} that the extra term is already of order $1/t$. Hence, the only parameterization requirement is $m\gtrsim \max\{\|w_t-w_0\|^4,\|w_t^\negi-w_0\|^4\}\approx \log^4(t).$ 
While the above describes our main technical novelty compared to \cite{richards2021learning}, our results surpass theirs in other aspects. Specifically, we also obtain tighter bounds on the optimization error, again thanks to leveraging self-bounded properties of the logistic loss. Overall, for the separable setting, we show a $\tilde{O}(1/n)$ test-loss bound compared to $O(T/n)$ in their paper. 

In closing, we remark that our logarithmic width requirements and expansiveness bounds are also significantly tighter than those that appear in \cite{richards2021stability,leistability}. While their results are not directly comparable to ours as they only apply to square-loss functions, we reference them here for completeness: \cite{richards2021stability} upper-bounds the expansiveness term on the left-hand side of \eqref{eq:Hossein2} by $\lesssim \left(1+\eta\sqrt{\eta t}/\sqrt{m}\right)$ which requires $m\gtrsim t^3$ so that is of order $1+1/t$. More recently, \cite{leistability} slightly modifies their bound to $\lesssim \left(1+\eta({\eta t})^{3/2}/(n\sqrt{m})\right)$ which requires $m\gtrsim (\eta t)^5/n^2.$

\end{document}